\documentclass[lettersize,journal]{IEEEtran}

\usepackage[utf8]{inputenc} 
\usepackage[T1]{fontenc}    
\usepackage{hyperref}       
\usepackage{url}            
\usepackage{booktabs}       
\usepackage{amsfonts}       
\usepackage{nicefrac}       
\usepackage{microtype}      
\usepackage{xcolor}         

\usepackage{microtype}
\usepackage{graphicx}
\usepackage{subcaption}
\usepackage{booktabs} 
\usepackage{url}

\usepackage{hyperref}
\usepackage[numbers]{natbib}

\usepackage{amsmath}
\usepackage{amssymb}
\usepackage{mathtools}
\usepackage{amsthm}
\usepackage{multirow}
\usepackage{makecell}

\usepackage{algorithm}
\usepackage{algorithmic}

\usepackage[capitalize,noabbrev]{cleveref}

\theoremstyle{plain}
\newtheorem{theorem}{Theorem}

\theoremstyle{definition}
\newtheorem{definition}[theorem]{Definition}

\theoremstyle{remark}

\usepackage{placeins}

\begin{document}

\title{Be Bayesian by Attachments to Catch More Uncertainty}

\author{Shiyu Shen, Bin Pan, Tianyang Shi, Tao Li and Zhenwei Shi

    \thanks{This work was supported in part by the National Key Research and Development Program of China under Grant 2022YFA1003800 and Grant 2022ZD0160401; in part by the National Natural Science Foundation of China under Grant 62001251, Grant 62125102, and Grant 62272248; and in part by the Beijing-Tianjin-Hebei Basic Research Cooperation Project under Grant F2021203109. \emph{(Corresponding author: Bin Pan)}}

    \thanks{Shiyu Shen and Bin Pan (corresponding author) are with the School of Statistics and Data Science, KLMDASR, LEBPS, and LPMC, Nankai University, Tianjin 300071, China
        (e-mail: shenshiyu@mail.nankai.edu.cn; panbin@nankai.edu.cn).}
    \thanks{Tianyang Shi and Zhenwei Shi are with the Image Processing Center,
        School of Astronautics, Beihang University, Beijing 100191, China (e-mail: shitianyang@buaa.edu.cn; shizhenwei@buaa.edu.cn).}
    \thanks{Tao Li is with the College of Computer Science, Nankai University, Tianjin 300071, China (e-mail: litao@nankai.edu.cn).}
}

\markboth{IEEE Transactions on Neural Networks and Learning Systems}%
{Shell \MakeLowercase{\textit{et al.}}: A Sample Article Using IEEEtran.cls for IEEE Journals}

\maketitle

\begin{abstract}
    Bayesian Neural Networks (BNNs) have become one of the promising approaches for uncertainty estimation due to the solid theorical foundations. However, the performance of BNNs is affected by the ability of catching uncertainty. Instead of only seeking the distribution of neural network weights by in-distribution (ID) data, in this paper, we propose a new Bayesian Neural Network with an Attached structure (ABNN) to catch more uncertainty from out-of-distribution (OOD) data. We first construct a mathematical description for the uncertainty of OOD data according to the prior distribution, and then develop an attached Bayesian structure to integrate the uncertainty of OOD data into the backbone network. ABNN is composed of an expectation module and several distribution modules. The expectation module is a backbone deep network which focuses on the original task, and the distribution modules are mini Bayesian structures which serve as attachments of the backbone. In particular, the distribution modules aim at extracting the uncertainty from both ID and OOD data. We further provide theoretical analysis for the convergence of ABNN, and experimentally validate its superiority by comparing with some state-of-the-art uncertainty estimation methods Code will be public.
\end{abstract}


\section{Introduction}

Deep Neural Networks (DNNs) have gained widespread recognition as highly effective predictive models \cite{li2022research,zhang2022meta,lauriola2022introduction}. However, only remarkable predictive performance may not meet all the requirements of real-world applications. In some safety-critical scenarios, uncertainty estimation poses a significant challenge. \cite{gawlikowski2023survey,wang2022uncertainty,park2022uncertainty,oszkinat2022uncertainty}. Recent studies have raised concerns about the reliability of DNNs, as they tend to make overconfident predictions \cite{hein2019relu,wei2022mitigating}. Furthermore, when a DNN encounters out-of-distribution (OOD) samples that deviate significantly from its training data, it might generate overly confident yet meaningless predictions, resulting in potential issues. \cite{nguyen2015deep,wei2022mitigating}. Therefore, uncertainty estimation remains a challenge for DNNs.

Recently, Bayesian Neural Networks (BNNs) have shown promising performance for uncertainty estimation \cite{BBB, huang2023bayesian}. In BNNs, model parameters are treated as random variables with prior distributions, and their posterior distributions given training data are learned. During testing, BNNs generate predictions in the form of random variables with specific distributions. There are three major approaches to learn posterior distributions: Markov chain Monte Carlo \cite{li2016preconditioned,jia2023energy}, Laplacian approximation \cite{NEURIPS2021_a7c95857,gaedke2023parallelized} and variational inference \cite{BBB,nazaret2022variational}. Variational inference, i.e. approximating the true posterior with some simple distributions, is a popular approach \cite{graves2011practical}. For example, Blundell et~al. \cite{BBB} proposed a backpropagation-compatible algorithm for variational BNN training. Shridhar et~al. \cite{shridhar2019comprehensive} introduced variational Bayesian inference into Convolution Neural Networks. Kristiadi et~al. \cite{kristiadi2020being} found it sufficient to build a ReLU network with only a single Bayesian layer.

However, most BNNs only catch uncertainty from in-distribution (ID) training data, potentially constraining their effectiveness \cite{de2023value,kristiadi2022being}. The uncertainty quantification ability of BNNs comes from their posteriors of parameters given data. Since OOD data are often incomplete, the true posteriors given OOD data are unclear, making estimation even more difficult. Therefore, incorporating OOD data into Bayesian inference remains a challenge. Consequently, there are instances where BNNs underperform frequentist methods, especially in OOD detection task \cite{OE,kristiadi2022being,hendrycks2018deep}.

\begin{figure*}
    \centering
    \subcaptionbox{BBP}[0.24\linewidth]{
        \includegraphics[width=1\linewidth]{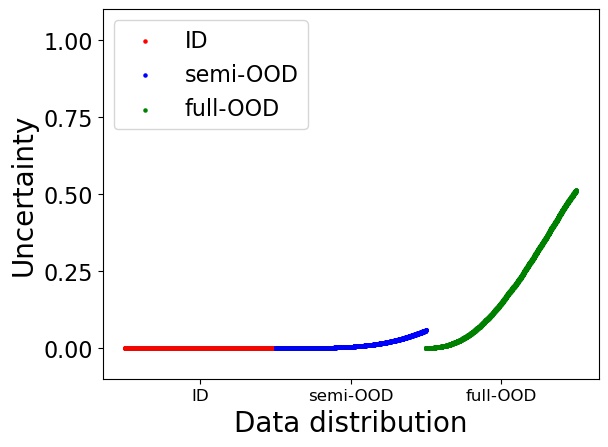}
    }
    \subcaptionbox{SDE-Net}[0.24\linewidth]{
        \includegraphics[width=1\linewidth]{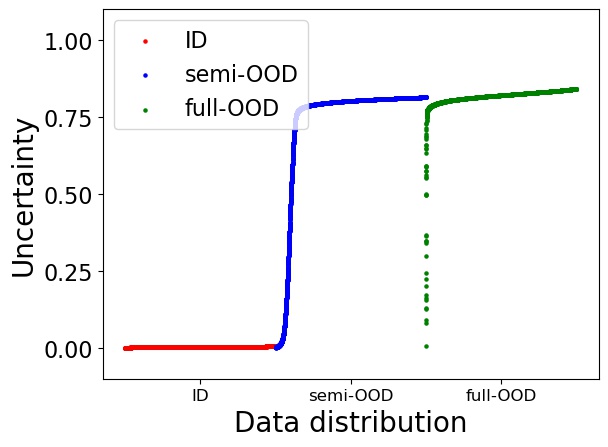}
    }
    \subcaptionbox{ABNN}[0.24\linewidth]{
        \includegraphics[width=1\linewidth]{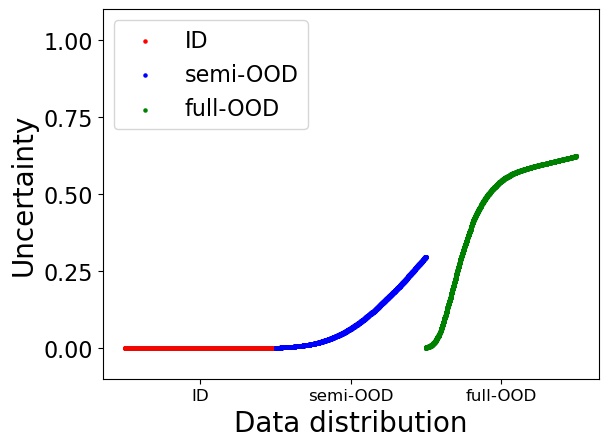}
    }
    \subcaptionbox{Ideal}[0.24\linewidth]{
        \includegraphics[width=1\linewidth]{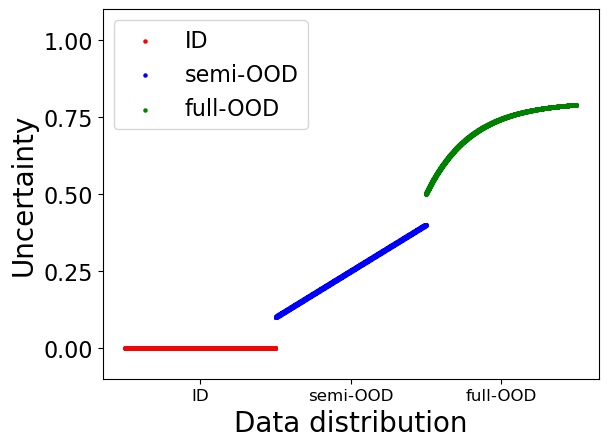}
    }
    \caption{Visualization of uncertainty for BBP \cite{BBB}, SDE-Net \cite{kong2020sde} and ABNN. The horizontal axis is the index of predictions ordered by uncertainty on each dataset, and the longitudinal axis is uncertainty. An ideal estimate should increase gradually as the input becomes more OOD, while still being separable on different datasets. Please refer to \cref{4_2} for more details.}
    \label{constraints}
\end{figure*}

OOD training has been well studied for OOD detection, and researchers find that using auxiliary outlier data can bring significant improvements \cite{hendrycks2018deep,OE}. For example, Outlier Exposure (OE) \cite{hendrycks2018deep} is a simple but effective method that uniformly labels OOD data for classification. Many variants based on OE have been proposed, and they achieved the state-of-the-art performance \cite{OE,zhang2023mixture,zhu2023openmix}. However, OE implicitly assumes uniform uncertainty across all OOD data by labeling them uniformly, which raises concerns. While the uniform uncertainty offers advantages in OOD detection, it may contradict the principle of uncertainty estimation.

Assigning the same uncertainty to all OOD data may result in misjudgments of valuable predictions. For instance, in handwritten digit recognition task, both a printed digit and an animal image would be considered OOD samples. The printed digit may be correctly classified, so its uncertainty should lie between handwritten digits and animal images. In such a scenario, the recognition system should provide a prediction for the printed digit with a notice, while rejecting the animal image and issuing a warning. However, in OOD detection task, both the printed digit and the animal image are recognized as OOD data, and they are all abandoned. Consequently, the modeling that uncertainty increases continuously as inputs become more dissimilar from ID data is more realistic. We show how different methods handle uncertainty in \cref{constraints}. Furthermore, directly adapting OE to BNNs is not appropriate. The posteriors would be questionable considering that the pseudo labels for OOD data differ greatly from the true labels.

In this paper, a new Bayesian Neural Network with an Attachment structure (ABNN) is proposed to catch more uncertainty from out-of-distribution data. First, mathematical descriptions for the uncertainty of OOD data are established based on prior distributions, and OOD data are categorized into semi-OOD and full-OOD subsets. Then,  we investigate the correlation between uncertainty and parameter variance. Additionally, an adversarial strategy is proposed to integrate OOD uncertainty into ID uncertainty. Meanwhile, we develop an attachment structure to mitigate the adverse effect of OOD training on the backbone network. ABNN consists of two modules: an expectation module and several distribution modules. The expectation module focuses on the primary task of the neural network, which is similar to the expectation of a traditional BNN. Distribution modules are mini Bayesian structures aiming to catch the uncertainty from both ID data and OOD data. The structure of distribution modules can be simpler than the expectation module, as researchers indicate that just a few Bayesian layers are enough to catch uncertainty \cite{kristiadi2020being,sharma2023bayesian}. Therefore, we design distribution modules to be mini-sized, which seem like attachments of the expectation module.

To summarize, our contributions are:
\begin{itemize}
    \item We propose a variational inference BNN framework to catch additional uncertainty from OOD data.
    \item We refine the definitions for OOD data, and propose the concept of semi-OOD and full-OOD data.
    \item We propose a new training approach to appropriately catch OOD uncertainty. The convergence is theoretically proved.
    \item We design a general attachment structure to maintain the predictive power of backbone while equipping it with better uncertainty estimation ability.
\end{itemize}

\section{Related Work}

\subsection{Bayesian neural network}
Bayesian neural network aims to estimate the uncertainty of parameters \cite{BBB,kristiadi2020being,jospin2022hands}. The key idea of BNN is to estimate the posterior distributions of parameters given training data. Recently, researches have proposed several realization methods for BNN, including Variational Inference\cite{BBB}, Markov chain Monte Carlo \cite{li2016preconditioned} and Laplace Approximation \cite{NEURIPS2021_a7c95857,liu2023bayesian}. During testing, different DNNs are sampled from the BNN posterior, and each DNN make a prediction. The final prediction of the BNN is determined by aggregating these individual predictions, with uncertainty being represented by the variance.

Variational inference is a popular approach to train BNNs. It approximates the true parameter posteriors using commonly used distributions, such as Gaussian distribution. The distance between variational distribution and the true posterior is quantified by Kullback-Leibler (KL) divergence. Blundell et al. \cite{BBB} propose a backpropagation-compatible algorithm for variational BNN training. Kristiadi et al. \cite{kristiadi2020being} find it sufficient to build a ReLU network with a single Bayesian layer. Krishnan et al. \cite{krishnan2020specifying} propose a method to choose informed weight priors from a DNN.

\subsection{Out-of-distribution detection}

Out-of-distribution detection aims to equip a deep learning model with the ability to detect anomalous distributed test samples from in-distribution samples. Hendrycks and Gimpel \cite{hendrycks2016baseline} revealed that deep learning methods naturally have the potential to detect OOD samples. In addition, a comprehensive OOD detection evaluation metric was proposed \cite{hendrycks2016baseline}. Some methods do not change the initial training procedure, but propose new scores to represent the OOD level \cite{odin,liu2020energy,lin2021mood,djurisic2022extremely}. Some methods add a new class or new branch in the classifier to represent the OOD class, combined with specifically designed training methods, such as leaving-out strategy \cite{vyas2018out}, adversarial training \cite{bitterwolf2020certifiably,choi2019novelty,chen2021atom} and data augmentation \cite{hein2019relu,thulasidasan2019mixup,hendrycks2022pixmix}. 

Outlier exposure \cite{OE} is a simple yet effective method. OE introduces outlier exposure datasets into ID training. The outlier datasets can consist of real-world datasets unrelated to the ID dataset, or pseudo OOD datasets generated from noise. OE does not change the network structure, and the OOD datasets is used by maximizing the cross-entropy loss. Additionally, Kristiadi et al. \cite{kristiadi2022being} developed comprehensive methods for integrating OE into BNN.

\section{Mathematical Description}
\label{2}

In this section, we will present some necessary definitions and properties. FFor clarity and simplicity, we limit our discussion to balanced classification tasks with $K$ classes, using MNIST \cite{lecun1998gradient} as an illustrative example. The definitions and properties can be easily extended to other situations.

\subsection{Segmentation of ID and OOD data}

\begin{figure}[htbp]
    \includegraphics[width=0.95\linewidth]{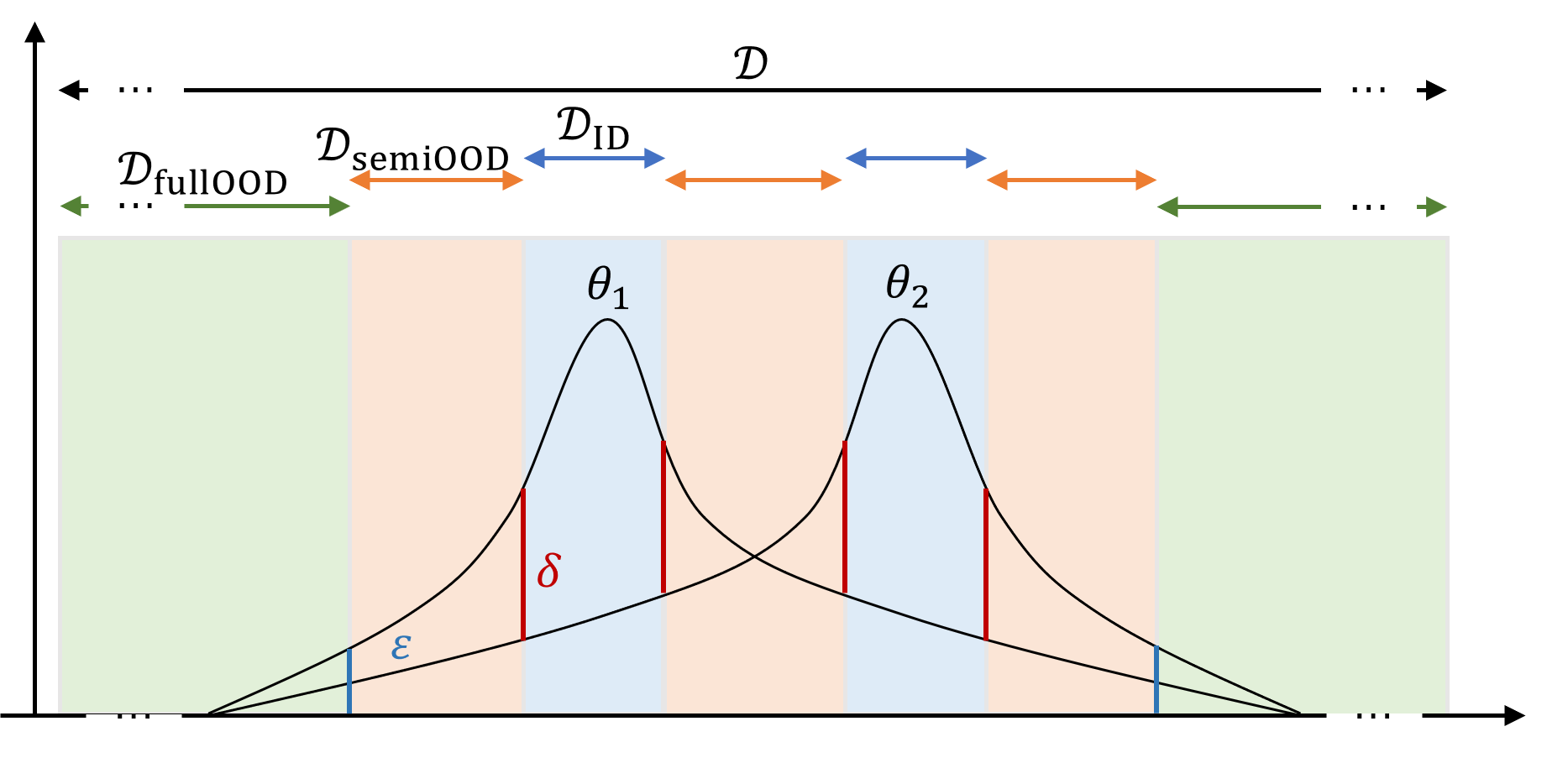}
    \caption{Segmentation of data space. Whether a data point belongs to a specific distribution determines the segmentation bound.\label{distribution_vis}}
\end{figure}

We first segment the whole data space $\mathcal{D}$ into subspaces $\mathcal{D}_{\rm{ID}}$,$\mathcal{D}_{\rm{semiOOD}}$ and $\mathcal{D}_{\rm{fullOOD}}$ to describe the property of different data. This segmentation is based on the assessment of that whether a sample unequivocally belongs to a certain distribution (low uncertainty), potentially belongs to it (moderate uncertainty), or clearly does not belong to it (high uncertainty).  However, it is the nature of the dataset that determines how we describe the data space through definitions, rather than that our definitions dictate data generation. Please refer to \cref{distribution_vis} for better understanding.

\begin{definition}[$\mathcal{D}_{\rm{ID}}$]
    \label{def1}
    Let $X$ denote the input, $X\in \mathcal{D}$. $\{f(X,\theta):\theta\in \Theta\}$ is the family of density functions on $\mathcal{D}$, $\theta$ is the parameter, $\Theta=\{\theta_1,\theta_2,...,\theta_K\}$ denotes the parameters corresponding to data distribution. Given $\delta>0$, $\varepsilon>0$, we call $\mathcal{D}_{\rm{ID}}:=\{X\in \mathcal{D}:\exists \theta_i\in\Theta,f(X,\theta_i)>\varepsilon \} \cap \{X\in \mathcal{D}:\exists\theta_i\in\Theta, \forall \theta_j\neq\theta_i$ and $\theta_j\in\Theta, f(X,\theta_i)>f(X,\theta_j)+\delta\}$ as in-distribution data space.
\end{definition}

For example, in MNIST digits classification, $\mathcal{D}$ is the collection of all possible $28\times 28$ images and $\mathcal{D}_{\rm{ID}}$ is MNIST. $\{f(X,\theta):\theta\in \Theta\}$ is the collection of all the MNIST data distributions. An image that not only exhibits a digit shape ($f(X,\theta_i)>\varepsilon$) but also aligns closely with the characteristic style of MNIST ($f(X,\theta_i)>f(X,\theta_j)+\delta$) should be recognized as a potential sample from MNIST. Here, $\delta$ and $\varepsilon$ represent abstract concepts that describe the degree of "belonging to a distribution." In practice, with only access to the training dataset, the data distribution can be empirically estimated. The segmentation into different subspaces, however, is determined by the inherent characteristics of the training datasets rather than manual selection of $\delta$ and $\varepsilon$. Additionally, an image displaying a digit shape but lacking clear adherence to the MNIST style ($f(X,\theta_i)\approx f(X,\theta_j)$) falls within the category of semi-OOD data, while an image devoid of a digit shape is classified as full-OOD data.

\begin{definition}[$\mathcal{D}_{\rm{semiOOD}}$]
    \label{def2}
    Following \cref{def1}, $\mathcal{D}_{\rm{semiOOD}}:=\{X\in \mathcal{D}:\exists \theta_i\in\Theta,f(X,\theta_i)>\varepsilon \}\setminus \mathcal{D}_{\rm{ID}}$.
\end{definition}
\begin{definition}[$\mathcal{D}_{\rm{fullOOD}}$]
    \label{def3}
    Following \cref{def1}, $\mathcal{D}_{\rm{fullOOD}}:=\{X\in \mathcal{D}:\forall \theta_i\in\Theta,f(X,\theta_i)\leqslant \varepsilon \}$.
\end{definition}
In MNIST digits classification, the collection of all digit images that deviate from the MNIST is $\mathcal{D}_{\rm{semiOOD}}$. These images exhibit similar distributions to ID data but possess distinct styles. A prime example is the SVHN \cite{netzer2011reading}, which serves as a subset of $\mathcal{D}_{\rm{semiOOD}}$. On the other hand, $\mathcal{D}_{\rm{fullOOD}}$ is all the $28\times 28$ images that are not digits.

\subsection{Data Distribution}

Following the segmentation process, it becomes feasible to designate specific distributions to represent ID, semi-OOD, or full-OOD data. Before segmentation, there exists a global prior distribution $p(X)$ for all data, but it is inaccessible. However, after segmentation, we can define conditional distributions on $\mathcal{D}_{\rm{ID}}$ and $\mathcal{D}_{\rm{semiOOD}}$ based on the constraint of $f(X,\theta)$.

\begin{definition}[$p(X|\theta,\mathcal{D}_{\rm{ID}})$]
    \label{def4}
    Following \cref{def1}, $\mathcal{D}_{\rm{ID}}^i := \{X\in \mathcal{D}_{\rm{ID}}: \forall \theta_j\neq\theta_i$ and $\theta_j\in\Theta, f(X,\theta_i)>f(X,\theta_j)+\delta\}$. Then we call
    \begin{align}
        G(f(X,\theta_i)):=\left\{
        \begin{aligned}
             & f(X,\theta_i)+\frac{\int_{\mathcal{D}\setminus \mathcal{D}_{\rm{ID}}^i} f(u,\theta_i) \,du}{\int_{\mathcal{D}_{\rm{ID}}^i}1 \,du},X\in\mathcal{D}_{\rm{ID}}^i \\
             & 0,X\in\mathcal{D}\setminus \mathcal{D}_{\rm{ID}}^i
        \end{aligned}
        \right.
    \end{align}
    the density function of $X$ given the condition that $X \in \mathcal{D}_{\rm{ID}}$ and belongs to $i_{th}$ class,  denoted as $p(X|\theta_i,\mathcal{D}_{\rm{ID}})$.
\end{definition}

We establish this definition based on a common assumption: In classification tasks, each class corresponds to a specific distribution, and the process of classifying a sample is to find the most probable distribution that this sample originates from \cite{wan2018rethinking,Xie_2022_CVPR}. For easy notation, we continue to refer to the prior distributions by $f$. If the class of $X$ is not specified, we can calculate its prior distribution by formula $p(X|\mathcal{D}_{\rm{ID}})=\sum_{i = 1}^{K}p(\theta_i)p(X|\theta_i,\mathcal{D}_{\rm{ID}})$. Meanwhile, we assume that $\theta=\theta_i$ corresponds to that $X$ belongs to the $\rm{i_{th}}$ class.

On $\mathcal{D}_{\rm{semiOOD}}$, we can define prior distributions similarly as \cref{def4}, where the domain of definition for each class is $\mathcal{D}_{\rm{semiOOD}}^i := \{X\in \mathcal{D}_{\rm{semiOOD}}: \forall \theta_j\neq\theta_i$ and $\theta_j\in\Theta, f(X,\theta_i)>f(X,\theta_j)\}$. In addition, there may be a special situation when the class is not clear $\{X\in \mathcal{D}_{\rm{semiOOD}}: \exists \theta_i\in\Theta$ and $\theta_j\in\Theta, f(X,\theta_i)=f(X,\theta_j)=\max_{\theta}f(X,\theta)\}$, and the density function on this domain is the truncation of $\max_{\theta}f(X,\theta)$.

On $\mathcal{D}_{\rm{fullOOD}}$, it is meaningless to consider the relationship between OOD samples and the distributions $f(X,\theta)$, because all the density values are too low to distinguish. In other words, $X$ is independent of $\theta$. We assume that on $\mathcal{D}_{\rm{fullOOD}}$, $X$ has the same distribution as their prior distribution, even though the precise formula for this distribution is unknown.
\begin{definition}[$p(X|\theta,\mathcal{D}_{\rm{fullOOD}})$]
    \label{def5}
    $p(X|\theta_i,\mathcal{D}_{\rm{fullOOD}}):=p(X|\mathcal{D}_{\rm{fullOOD}})$ is called the prior distribution of X given the condition that $X\in \mathcal{D}_{\rm{fullOOD}}$ and belongs to $i_{th}$ class. $p(X|\mathcal{D}_{\rm{fullOOD}})=\frac{P(\mathcal{D}_{\rm{fullOOD}}|X)p(X)}{P(\mathcal{D}_{\rm{fullOOD}})}=\mathbb{I}_{X\in\mathcal{D}_{\rm{fullOOD}}}\frac{m(\mathcal{D})}{m(\mathcal{D}_{\rm{fullOOD}})}p(X)$, $\mathbb{I}_{(\cdot)}$ is the indicator function, $m(\cdot)$ is a measure.
\end{definition}

\subsection{Relationship Between Variance and Uncertainty}

With the above descriptions, we can validate two of our core hypotheses: (1) OOD Labels have higher variances compared to ID labels, and (2) higher variances of BNN layers correspond to higher uncertainty. We only discuss these hypotheses for networks with softmax classification layers.

We begin by demonstrating that although the exact distributions of OOD labels are unknown, if we simplify them as one-hot encoded, their variances are higher than those of ID labels. A one-hot label can be regarded as the Maximum A Posteriori estimate of a categorical distribution \cite{wang2019classification,le2023uncertainty}. To check the variance of a label $y\in\{[1,0,...,0],[0,1,...,0],...,[0,0,...,1]\}$, we turn to each dimension of $y$ and take $y[i]$ as Bernoulli distributed.

\begin{theorem}
    \label{th1}
    $var(y[i]|X,\mathcal{D}_{\rm{ID}})<var(y[i]|X,\mathcal{D}_{\rm{fullOOD}})$
\end{theorem}

\begin{proof}
    According to Bayes Rule,
    \begin{align}
          & P(y[i]=1|X,\mathcal{D}_{\rm{ID}})                                                                                                            \\
        = & \frac{P(X|\theta_i,\mathcal{D}_{\rm{ID}})P(\theta_i|\mathcal{D}_{\rm{ID}})}{P(X|\mathcal{D}_{\rm{ID}})}                                      \\
        = & \frac{P(X|\theta_i, \mathcal{D}_{\rm{ID}})P(\theta_i|\mathcal{D}_{\rm{ID}})}{\sum_{n=1}^{K} P(X|\theta_n,\mathcal{D}_{\rm{ID}})P(\theta_n)}.
    \end{align}
    Because of the balance assumption,
    \begin{equation}
        P(\theta_n)=P(\theta_n|\mathcal{D}_{\rm{ID}})=\frac{1}{K},
    \end{equation}
    \begin{equation}
        \implies P(y[i]=1|X,\mathcal{D}_{\rm{ID}})=\frac{f(X,\theta_i)}{\sum_{n=1}^{K}f(X,\theta_n)}.
    \end{equation}
    Similarly, $P(y[i]=1|X,\mathcal{D}_{\rm{fullOOD}})=\frac{1}{K}$.\\
    According to our definition, on $\mathcal{D}_{\rm{ID}}$,
    \begin{equation}
        \exists i,\forall j\neq i,f(X,\theta_i)>f(X,\theta_j)+\delta,
    \end{equation}
    \begin{align}
        \implies &\frac{f(X,\theta_i)}{\sum_{n = 1}^{K} f(X,\theta_n)}>\frac{1+\delta}{K+\delta}>\frac{1}{K},\\
        &\frac{f(X,\theta_j)}{\sum_{n = 1}^{K} f(X,\theta_n)}<\frac{1}{K+\delta}<\frac{1}{K},\forall j\neq i.
    \end{align}
    Since $var(y[i])=P(y[i]=1)(1-P(y[i]=1))$,
    \begin{align}
        var(y[i]|X,\mathcal{D}_{\rm{ID}})<var(y[i]|X,\mathcal{D}_{\rm{fullOOD}}),\forall i={1,2,...,K}.
    \end{align}
\end{proof}

The proof can also indicate that on $\mathcal{D}_{\rm{ID}}$ and $\mathcal{D}_{\rm{semiOOD}}$, labels of misclassified samples have higher variances. With \cref{th1}, we believe that although the explicit labels of OOD data are unclear, their variance is higher than those of ID data. It can also be proved that $var(y[i]|\mathcal{D}_{\rm{ID}})<var(y[i]|\mathcal{D}_{\rm{semiOOD}})<var(y[i]|\mathcal{D}_{\rm{fullOOD}})$.

However, we only distinguish between $\mathcal{D}_{\rm{semiOOD}}$ and $\mathcal{D}_{\rm{fullOOD}}$ during testing, and during training we only use singular OOD dataset. This approach is practical as it aligns with typical scenarios where labeled  $\mathcal{D}_{\rm{ID}}$ is available, and $\mathcal{D}_{\rm{fullOOD}}$ can be constructed by using unrelated and comprehensive datasets such as CIFAR10 \cite{krizhevsky2009learning}, or by adding noise to $\mathcal{D}_{\rm{ID}}$. However, obtaining existing datasets that are uniformly distributed on $\mathcal{D}_{\rm{semiOOD}}$ is challenging, and creating a pseudo $\mathcal{D}_{\rm{semiOOD}}$ that appropriately bridges the gap between $\mathcal{D}_{\rm{ID}}$ and $\mathcal{D}_{\rm{fullOOD}}$ is also difficult \cite{hsu2020generalized}. Despite these challenges, ABNN can still effectively handle $\mathcal{D}_{\rm{semiOOD}}$ through its adversarial training procedure.

\begin{theorem}
    \label{th2}
    As $\sigma\rightarrow +\infty$, $P(softmax(x_1+\epsilon_1\cdot \sigma,x_2+\epsilon_2\cdot \sigma)[1]>\frac{1}{2})\rightarrow\frac{1}{2}$, where $\epsilon_1$ and $\epsilon_2$ are independent standard Gaussian noises.
\end{theorem}

\begin{proof}
    \begin{align}
                        & \frac{exp\{x_1+\epsilon_1\cdot \sigma\}}{exp\{x_1+\epsilon_1\cdot \sigma\}+exp\{x_2+\epsilon_2\cdot \sigma\}}>\frac{1}{2} \\
        \Leftrightarrow & \ exp\{x_1+\epsilon_1\cdot \sigma\}>exp\{x_2+\epsilon_2\cdot \sigma\}                                                     \\
        \Leftrightarrow & \ x_1+\epsilon_1\cdot \sigma>x_2+\epsilon_2\cdot \sigma                                                                   \\
        \Leftrightarrow & \ \epsilon_2-\epsilon_1<(x_1-x_2)/\sigma.
    \end{align}
    Since $\epsilon_1$ and $\epsilon_2$ are independent Gaussian random variables, $\epsilon_2-\epsilon_1$ is also Gaussian distributed.\\
    Denote the variance of $\epsilon_2-\epsilon_1$ to be $\sigma_0^2$, then
    \begin{equation}
        P(\epsilon_2-\epsilon_1<(x_1-x_2)/\sigma)=\Phi((x_1-x_2)/(\sigma\cdot\sigma_0)),
    \end{equation}
    where $\Phi(\cdot)$ is the distribution function of a standard Gaussian distribution.\\
    As $\sigma\rightarrow +\infty$,
    \begin{equation}
        (x_1-x_2)/(\sigma\cdot\sigma_0)\rightarrow 0,
    \end{equation}
    \begin{equation}
        \implies \Phi((x_1-x_2)/(\sigma\cdot\sigma_0))\rightarrow \frac{1}{2},
    \end{equation}
    which means
    \begin{equation}
        P(softmax(x_1+\epsilon_1\cdot \sigma,x_2+\epsilon_2\cdot \sigma)[1]>\frac{1}{2})\rightarrow\frac{1}{2}.
    \end{equation}
\end{proof}

The proof can be easily extended to higher dimension. \cref{th2} indicates that if we treat the parameters of a Gaussian variational inference BNN as common weights plus noises, higher noise levels typically lead to higher uncertainty, regardless of the parameter values. In general, augmenting larger noises during forward propagation tends to result in larger variances of features and final outputs \cite{kong2020sde}. As an application, we can represent the uncertainty from $\mathcal{D}_{\rm{OOD}}$ by increasing the variance of certain activated Bayesian parameters while maintaining their expectations unchanged.

\section{Implementation}

\begin{figure*}[htbp]
    \centering
    \includegraphics[width=0.95\textwidth]{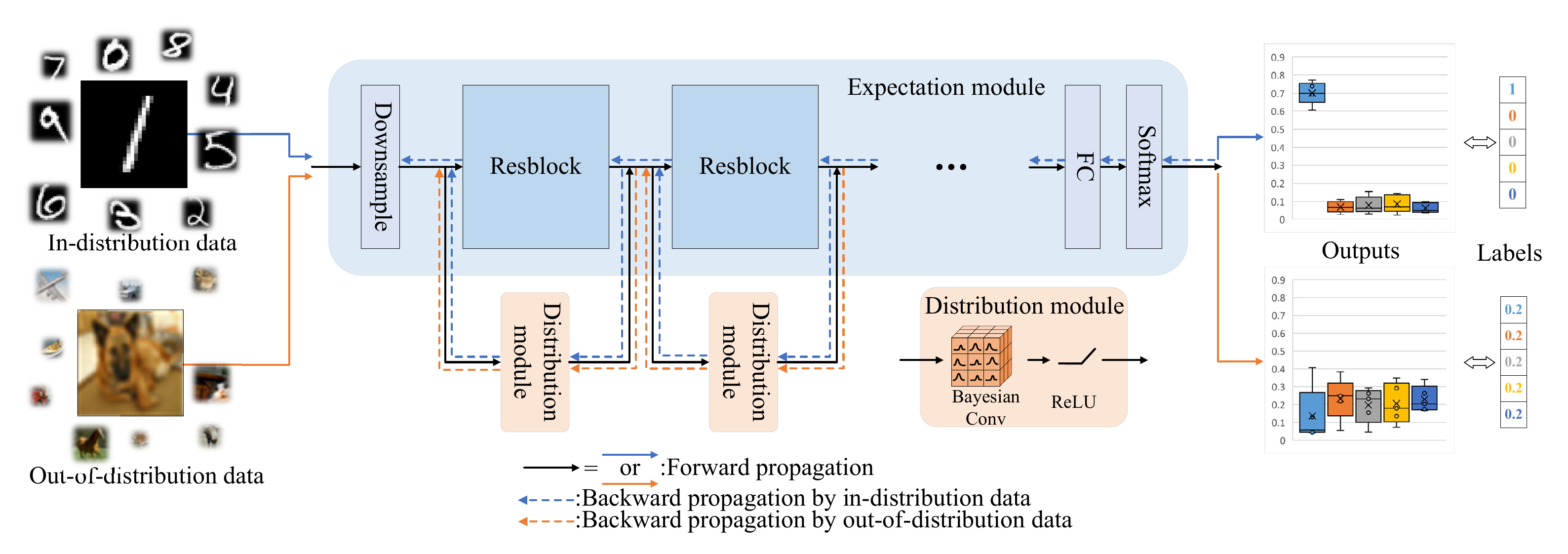}
    \caption{Components of ABNN. We take ResNet as the backbone and attach a distribution module to each Resblock. For out-of-distribution data, our network can catch its uncertainty and generate probabilistic results with large variances. For in-distribution data, our network can predict results with small variances.}
    \label{structure}
\end{figure*}

Our objective is to train a Gaussian variational inference BNN that meets the following criteria: (1) It accurately predicts ID data while adding appropriately small noises during forward propagation. (2) It adds large noises during forward propagation for OOD data, resulting in predictions with high variance. (3) It maintains high prediction accuracy for semi-OOD data. To achieve this, we begin by specifying pseudo labels for OOD data to derive pseudo posteriors and identify the most suitable training method. However, directly applying our training method may harm the prediction accuracy, so an attachment structure is proposed. Furthermore, we demonstrate that despite introducing OOD training, our approach remains consistent with typical BNN frameworks, as under certain assumptions, our objective functions can be interpreted as conventional ones \cite{BBB}.

\subsection{Method to Use OOD Data}
\label{3_1}

According to \cref{2}, the variances of pseudo labels (still denoted as $y$) for OOD data should be properly large. One common idea is to sample from $\rm{U}[0,1]$ to get $z_1,z_2,...,z_K$ and take their normalization $[\frac{z_1}{\sum_{i = 1}^{K} z_i},\frac{z_2}{\sum_{i = 1}^{K} z_i},...,\\\frac{z_K}{\sum_{i = 1}^{K} z_i}]$ as a label \cite{maennel2020neural}. While this approach meets our requirements, using random labels poses three challenges: (1) The same OOD sample will have different labels at different time, leading to instable training process; (2) This approach manually defines the variance of outputs, while the variance of true labels is unknown; (3) It can hardly be generalized to regression tasks.

As a result, we do not adopt random pseudo labels. Instead, we follow another common idea: the mean of all possible label values \cite{hendrycks2018deep}. For classification tasks with $K$ classes, our pseudo label is $[\frac{1}{K},\frac{1}{K},...,\frac{1}{K}]$, and for regression, our pseudo label is the mean value of all possible predictions.

However, the pseudo labels are differently handled compared to OE methods. If adopted directly, there is no guarantee that greater noises will be added for more OOD data, potentially resulting in the BNN failing to distinguish between different types of OOD data. To handle this problem, contrary to traditional BNN, we aim to increase the variances of certain parameters by maximizing the Kullback-Leibler (KL) divergence between variational distributions and pseudo posterior distributions on $\mathcal{D}_{\rm{OOD}}$  (still denoted as $p(\omega|\mathcal{D}_{\rm{OOD}})$). We first show that this operation does increase the variances of OOD predictions. By \cref{th2}, we can check that as the variance of parameters approaches to $+\infty,P(y[i]>y[j])\rightarrow\frac{1}{2},\forall i\neq j \in {1,2,...,K}$. Considering that $\sum_{i=1}^{K}y[i]=1$, so $\mathbb{E}[y[i]]\rightarrow\frac{1}{K}$. What remains to be validated is that the variances of some parameters can approach $+\infty$:
\begin{theorem}
    \label{assumption1}
    Assume that $q(\omega|\mu,\sigma^2)=\mathcal{N}(\mu,\sigma^2)$ and the prior distribution of $\omega$ is $\mathcal{N}(0,1)$. If $\sigma^2\geqslant 1$, $\mathop{\arg\max}\limits_{\sigma^2}\mathbb{D}_{KL}[q(\omega|\mu,\sigma^2)||p(\omega|\mathcal{D}_{\rm{OOD}})]=+\infty.$
\end{theorem}

\begin{proof}
    We take the variational distribution to be $\mathcal{N}(\mu,\sigma^2)$.
    \begin{align}
                & \mathbb{D}_{KL}[q(\omega|\mu,\sigma^2)||p(\omega|\mathcal{D}_{\rm{OOD}})]                                                                  \\
        =       & \int q(\omega|\mu,\sigma^2)log(\frac{q(\omega|\mu,\sigma^2)}{p(\omega|\mathcal{D}_{\rm{OOD}})})\,d\omega                                   \\
        =       & \int q(\omega|\mu,\sigma^2)log(\frac{q(\omega|\mu,\sigma^2)p(\mathcal{D}_{\rm{OOD}})}{p(\mathcal{D}_{\rm{OOD}}|\omega)p(\omega)})\,d\omega \\
        =       & \mathbb{D}_{KL}[q(\omega|\mu,\sigma^2)||p(\omega)]+\int q(\omega|\mu,\sigma^2)log(p(\mathcal{D}_{\rm{OOD}}))\,d\omega                      \\
                & -\int q(\omega|\mu,\sigma^2)log(p(\mathcal{D}_{\rm{OOD}}|\omega))\,d\omega                                                                 \\
        \approx & \mathbb{D}_{KL}[q(\omega|\mu,\sigma^2)||p(\omega)]+\int q(\omega|\mu,\sigma^2)log(p(\mathcal{D}_{\rm{OOD}}))\,d\omega                      \\
                & -\sum_{i = 1}^{n} log(p(\mathcal{D}_{\rm{OOD}}|\omega_i)),
    \end{align}
    where $\omega_1,\omega_2,...,\omega_n$ are independent samples form $q(\omega|\mu,\sigma^2)$.\\
    The second term is a constant, so
    \begin{align}
          & \mathop{\arg\max}\limits_{\sigma^2}\mathbb{D}_{KL}[q(\omega|\mu,\sigma^2)||p(\omega|\mathcal{D}_{\rm{OOD}})]                                    \\
        \approx  & \mathop{\arg\max}\limits_{\sigma^2}\mathbb{D}_{KL}[q(\omega|\mu,\sigma^2)||p(\omega)]-\sum_{i = 1}^{n} log(p(\mathcal{D}_{\rm{OOD}}|\omega_i)).
    \end{align}

    Given that the prior distribution of $\omega$ is $\mathcal{N}(0,1)$, which means $p(\omega)=\frac{1}{\sqrt{2\pi}}exp\{-\omega^2\}$, and $q(\omega|\mu,\sigma^2)=\frac{1}{\sqrt{2\pi\sigma^2}}exp\{-\frac{(\omega-\mu)^2}{\sigma^2}\}$, it can be easily checked that given $\mu=0$ and denote $\omega=\mu+\sigma\epsilon$,$\epsilon\sim \mathcal{N}(0,1)$,
    \begin{align}
        \frac{\partial \mathbb{D}_{KL}[q(\omega|\mu,\sigma^2)||p(\omega)]}{\partial \sigma^2}=\epsilon^2(\sigma-\frac{1}{\sigma})+\frac{(\epsilon^2-1)}{2}\frac{1}{\sigma^2}.
    \end{align}

    Taking expectation of $\epsilon$, we get
    \begin{align}
        \frac{\partial \mathbb{D}_{KL}[q(\omega|\mu,\sigma^2)||p(\omega)]}{\partial \sigma^2}=\sigma-\frac{1}{\sigma}.
    \end{align}

    At the beginning, $(\mu,\sigma^2)$ is initialized as $(0,1)$, and $\mathbb{D}_{KL}[q(\omega|\mu,\sigma^2)||p(\omega)]\rightarrow \infty$ in probability for both $\sigma\downarrow 0$ and $\sigma\uparrow +\infty$. If $\sigma^2>1$, $\sigma^2$ will approach to $+\infty$, so how the full loss formula reaches maximum is dependent on $-\sum_{i = 1}^{n} log(p(\mathcal{D}_{\rm{OOD}}|\omega_i))$.

    However, due to the unspecified form of $-\sum_{i = 1}^{n} log(p(\mathcal{D}_{\rm{OOD}}|\omega_i))$, which represents the Cross Entropy Loss ($L_{CE}$), we can not prove that when $\sigma^2 = 1$, $\frac{\partial -\sum_{i = 1}^{n} log(p(\mathcal{D}_{\rm{OOD}}|\omega_i))}{\partial \sigma^2}>0$, but we can show that as $\sigma^2\rightarrow+\infty$,
    \begin{align}
        P\{ & L_{CE}[(\frac{e^{x_1+\sigma\epsilon_1}}{e^{x_1+\sigma\epsilon_1}+e^{x_2+\sigma\epsilon_2}},\frac{e^{x_2+\sigma\epsilon_2}}{e^{x_1+\sigma\epsilon_1}+e^{x_2+\sigma\epsilon_2}}),(\frac{1}{2},\frac{1}{2})]
        \\>&L_{CE}[(\frac{e^{x_1}}{e^{x_1}+e^{x_2}},\frac{e^{x_2}}{e^{x_1}+e^{x_2}}),(\frac{1}{2},\frac{1}{2})]\}\rightarrow 1
    \end{align}
    where $\epsilon_1$ and $\epsilon_2$ are Gaussian distributed random variables with zero mean. Assume $x_1>x_2$, the possibility can be rewritten as
    \begin{align}
          & P[\frac{1}{2}log(\frac{e^{x_1+\sigma\epsilon_1}}{e^{x_1+\sigma\epsilon_1}+e^{x_1+\sigma\epsilon_1}})+\frac{1}{2}log(\frac{e^{x_2+\sigma\epsilon_2}}{e^{x_1+\sigma\epsilon_1}+e^{x_1+\sigma\epsilon_1}}) \\
          & >\frac{1}{2}log(\frac{e^{x_1}}{e^{x_1}+e^{x_2}})+\frac{1}{2}log(\frac{e^{x_1}}{e^{x_1}+e^{x_2}})]                                                                                                       \\
        = & P(\frac{e^{x_1+\sigma\epsilon_1}}{e^{x_1+\sigma\epsilon_1}+e^{x_2+\sigma\epsilon_2}}>\frac{e^{x_1}}{e^{x_1}+e^{x_2}})                                                                                   \\
        + & P(\frac{e^{x_1+\sigma\epsilon_1}}{e^{x_1+\sigma\epsilon_1}+e^{x_2+\sigma\epsilon_2}}<\frac{e^{x_2}}{e^{x_1}+e^{x_2}}).
    \end{align}

    Simplify the former term, we get
    \begin{align}
          & P(\frac{e^{x_1+\sigma\epsilon_1}}{e^{x_1+\sigma\epsilon_1}+e^{x_2+\sigma\epsilon_2}}>\frac{e^{x_1}}{e^{x_1}+e^{x_2}}) \\
        = & P(\frac{1}{e^{x_1}+e^{x_2+\sigma(\epsilon_2-\epsilon_1)}}>\frac{1}{e^{x_1}+e^{x_2}})                                  \\
        = & P(\sigma(\epsilon_2-\epsilon_1)<0)                                                                                    \\
        = & \frac{1}{2},
    \end{align}
    and simplify the latter term, we get
    \begin{align}
          & P(\frac{e^{x_1+\sigma\epsilon_1}}{e^{x_1+\sigma\epsilon_1}+e^{x_2+\sigma\epsilon_2}}<\frac{e^{x_2}}{e^{x_1}+e^{x_2}}) \\
        = & P(\frac{1}{1+e^{x_2-x_1+\sigma(\epsilon_2-\epsilon_1)}}<\frac{1}{e^{x_1-x_2}+1})                                      \\
        = & P(\epsilon_2-\epsilon_1>\frac{2(x_1-x_2)}{\sigma})\uparrow \frac{1}{2}.
    \end{align}

    Combine these simplifications, we find as $\sigma^2$ become greater, the probability that $-\sum_{i = 1}^{n} log(p(\mathcal{D}_{\rm{OOD}}|\omega_i))$ increases will become larger; and if $\sigma^2\rightarrow +\infty$, $-\sum_{i = 1}^{n} log(p(\mathcal{D}_{\rm{OOD}}|\omega_i))$ increases with probability $1$.

    In contrast, if $\sigma^2\rightarrow0$,
    \begin{align}
        P\{ & L_{CE}[(\frac{e^{x_1+\sigma\epsilon_1}}{e^{x_1+\sigma\epsilon_1}+e^{x_2+\sigma\epsilon_2}},\frac{e^{x_2+\sigma\epsilon_2}}{e^{x_1+\sigma\epsilon_1}+e^{x_2+\sigma\epsilon_2}}),(\frac{1}{2},\frac{1}{2})]
        \\>&L_{CE}[(\frac{e^{x_1}}{e^{x_1}+e^{x_2}},\frac{e^{x_2}}{e^{x_1}+e^{x_2}}),(\frac{1}{2},\frac{1}{2})]\}\rightarrow 0
    \end{align}

    To conclude, $+\infty$ is the solution of $\mathop{\arg\max}\limits_{\sigma^2}\mathbb{D}_{KL}[q(\omega|\mu,\sigma^2)||p(\omega|\mathcal{D}_{\rm{OOD}})]$ in probability.
\end{proof}

We show later in \cref{3_2} that combined with ID training, eventually our OOD training can catch the right amount of uncertainty from OOD data.

\subsection{Attachment Structure}
\label{3_2}

Our objective is to train our BNN on OOD data by maximizing $\mathbb{D}_{KL}[q(\omega|\mu,\sigma^2)||p(\omega|\mathcal{D}_{\rm{OOD}})]$. However, this operation not only increases $\sigma^2$ but also changes $\mu$. If we employ this methodology to train a traditional BNN, the classification accuracy on ID data will decline, as demonstrated by our ablation experiment. To preserve high performance on ID data, we aim to separate the prediction ability and uncertainty estimation ability of a BNN by designing an attached structure. ABNN is composed of two types of modules: an expectation module and several distribution modules. For clarity, we illustrate our approach using ResNet as the backbone model. One direct advantage of the attachment structure is its flexibility: the backbone model can be replaced with any network architecture while the uncertainty estimation ability is determined solely by the attachments. We demonstrate the performance of ABNN with different backbones in \cref{backbone}, and the structure is depicted in \cref{structure}.

The expectation module is constructed by common network layers, responsible for the original task. We have no constraints for the expectation module on OOD data, but we believe that it can classify semi-OOD data well due to the generalization ability of DNNs.

The distribution modules are constructed by Bayesian layers \cite{BBB}, fitting the posteriors of some specific layers given both ID data and OOD data. On $\mathcal{D}_{\rm{ID}}$, they catch uncertainty by minimizing the KL-divergence between variational distributions and the true posterior distributions. On $\mathcal{D}_{\rm{OOD}}$, they catch uncertainty by maximizing the variance of certain parameters. Researchers show that just a few Bayesian layers are sufficient to catch enough uncertainty \cite{kristiadi2020being}, so we only simulate the distributions of some specific layers. As a result, ABNN only contains a few Bayesian modules that are ``attached'' to a common DNN, demanding only a bit more resources.

The objective function for training ABNN is:
\begin{align}
      & \mathop{\min}\limits_{\omega_1}\mathbb{E}_{x\sim\mathcal{D}_{\rm{ID}}}[\mathbb{E}[L(x)]]                             \\
    + & \mathop{\min}\limits_{\omega_2}\mathbb{D}_{KL}[q(\omega|\mu,\sigma^2)||p(\omega|\mathcal{D}_{\rm{ID}})]              \\
    + & \mathop{\max}\limits_{\omega_2}\alpha\cdot\mathbb{D}_{KL}[q(\omega|\mu,\sigma^2)||p(\omega|\mathcal{D}_{\rm{OOD}})],
\end{align}
where $L$ is the loss function that depends on the task, $\omega_1$ is the parameters of the expectation module, $\omega_2$ is the parameters of distribution modules and $\alpha$ is a hyperparameter that guarantees the performance of ABNN on $\mathcal{D}_{\rm{ID}}$. $\alpha$ does not have a significant effect on the performance of ABNN, and we chose $\alpha=0.95$ for experiments. Detailed discussions of $\alpha$ is shown in \cref{alpha}. We only use one kind of OOD data for training.

During each iteration, we train ABNN in three steps: (1) First, distribution modules are frozen and $\mathbb{E}_{x\sim\mathcal{D}_{\rm{ID}}}[\mathbb{E}[L(x)]]$ is minimized; (2) Then, the expectation module is frozen, and we minimize $\mathbb{D}_{KL}[q(\omega|\mu,\sigma^2)||p(\omega|\mathcal{D}_{\rm{ID}})]$; (3) At last, we keep freezing the expectation module and maximize $\alpha\cdot\mathbb{D}_{KL}[q(\omega|\mu,\sigma^2)||p(\omega|\mathcal{D}_{\rm{OOD}})]$.

Our training procedure is similar to an EM approach \cite{li2023self,zhou2020deep}: We partition the parameters of BNNs from $\omega=(\mu,\sigma^2)$ into $(\omega_1,\omega_2)=((\mu_1 ,\sigma_1^2\mu_2),(\mu_2,\sigma_2^2))$. We try to determine the expectation $\mu_1$ of $\omega_1$ as the latent variables ensuring predictive ability, and identify $(\mu_2,\sigma_2^2)$ as a balanced distribution between $p(\omega|\mathcal{D}_{\rm{ID}})$ and $p(\omega|\mathcal{D}_{\rm{OOD}})$. EM approaches are sensitive to the initial value of latent variables, but ABNN can be well initialized by pre-training. Our training process is summarized in \cref{alg}.

\begin{algorithm}[htbp]
    \caption{Training of ABNN. $f_\omega(\cdot)$ is the whole network, $\omega =  (\omega_1,\omega_2)$, $\omega_1$ represents parameters of expectation modules, $\omega_2=(\mu,\sigma^2)$ represents parameters of distribution modules, $N$ is the number of forward propagation times, L is a loss function.\label{alg}}
    \begin{algorithmic}
        \STATE Initialize $f_\omega(\cdot)$
        \FOR{training iterations}
        \STATE Sample minibatch data $X_{\rm{ID}}$ from $\mathcal{D}_{\rm{ID}}$
        \FOR{i=1; i$\leqslant $N; i++}
        \STATE Sample $\epsilon^i$ from $\mathcal{N}(0,1)$
        \STATE $\omega_2^i=\mu+\sigma\epsilon^i$
        \STATE $\omega^i=(\omega_1,\omega_2^i)$
        \STATE Update $\omega_1$ by $\mathop{\min}L(f_{\omega^i}(X_{\rm{ID}}))$
        \ENDFOR
        \STATE Update $\omega_2$ by $\mathop{\min}\sum_{i = 1}^{N}log(q(\omega^i|\mu,\sigma^2))-log(p(\omega^i))-log(p(X_{\rm{ID}}|\omega^i))  $
        \STATE Sample minibatch data $X_{\rm{OOD}}$ from $\mathcal{D}_{\rm{OOD}}$
        \FOR{i=1; i$\leqslant $N; i++}
        \STATE Sample $\epsilon^i$ from $\mathcal{N}(0,1)$
        \STATE $\omega_2^i=\mu+\sigma\epsilon^i$
        \STATE $\omega^i=(\omega_1,\omega_2^i)$
        \ENDFOR
        \STATE Update $\omega_2$ by $\mathop{\max}\alpha\sum_{i = 1}^{N}log(q(\omega^i|\mu,\sigma^2))-log(p(\omega^i))-log(p(X_{\rm{OOD}}|\omega^i))$
        \ENDFOR
    \end{algorithmic}
\end{algorithm}

Our training procedure can catch the appropriate amount of uncertainty through the adversarial principle between ID training and OOD training. We hope to increase the variance of certain parameters to represent OOD uncertainty while maintaining performance on ID data, which means $q(\omega|\mu,\sigma^2)$ should still fit $p(\omega|\mathcal{D}_{\rm{ID}})$ well. The formulation of the third loss term mirrors that of the second one and is governed by a hyperparameter. Consequently, the rate of parameter adjustments during OOD training remains proportional to that during ID training. As a result, only parameters with minimal impact on $\mathbb{D}_{KL}[q(\omega|\mu,\sigma^2)||p(\omega|\mathcal{D}_{\rm{ID}})]$ (e.g. convs that fail to catch meaningful features from ID data) can be primarily optimized by OOD loss. During testing, the more OOD features an input includes, the more parameters with high variance (e.g. convs that only catch OOD features) will be activated, resulting in increased uncertainty. Furthermore, we show in \cref{reinterpretation} that the two latter loss terms can still be interpreted as a traditional BNN loss.

\subsection{Reinterpretation of Loss Terms}
\label{reinterpretation}

Although ABNN is trained on ID data and OOD data separately, we show that it is still under the Bayesian Learning framework. The second and third loss terms can be merged into one KL divergence.

From \cref{assumption1}, we can get
\begin{align}
      & \mathop{\arg\max}\limits_{\mu,\sigma^2}\mathbb{D}_{KL}[q(\omega|\mu,\sigma^2)||p(\omega|\mathcal{D}_{\rm{OOD}})] \\
    = & \mathop{\arg\max}\limits_{\mu,\sigma^2}\mathbb{D}_{KL}[q(\omega|\mu,\sigma^2)||p(\omega)]                        \\
      & -\int q(\omega|\mu,\sigma^2)log(p(\mathcal{D}_{\rm{OOD}}|\omega))\,d\omega.
\end{align}

Similar,
\begin{align}
      & \mathop{\arg\min}\limits_{\mu,\sigma^2}\mathbb{D}_{KL}[q(\omega|\mu,\sigma^2)||p(\omega|\mathcal{D}_{\rm{ID}})] \\
    = & \mathop{\arg\min}\limits_{\mu,\sigma^2}\mathbb{D}_{KL}[q(\omega|\mu,\sigma^2)||p(\omega)]                       \\
      & -\int q(\omega|\mu,\sigma^2)log(p(\mathcal{D}_{\rm{ID}}|\omega))\,d\omega.
\end{align}

The first term works like a normalization and the second term is the initial loss function (e.g. cross entropy).
\begin{align}
      & \mathop{\min}\limits_{\mu,\sigma^2}\mathbb{D}_{KL}[q(\omega|\mu,\sigma^2)||p(\omega|\mathcal{D}_{\rm{ID}})]                                     \\
      & + \mathop{\max}\limits_{\mu,\sigma^2}\alpha\cdot\mathbb{D}_{KL}[q(\omega|\mu,\sigma^2)||p(\omega|\mathcal{D}_{\rm{OOD}})]                       \\
    = & \mathop{\min}\limits_{\mu,\sigma^2}\mathbb{D}_{KL}[q(\omega|\mu,\sigma^2)||p(\omega|\mathcal{D}_{\rm{ID}})]                                     \\
      & - \alpha\cdot\mathbb{D}_{KL}[q(\omega|\mu,\sigma^2)||p(\omega|\mathcal{D}_{\rm{OOD}})]                                                          \\
    = & \mathop{\min}\limits_{\mu,\sigma^2}(1-\alpha)\cdot\mathbb{D}_{KL}[q(\omega|\mu,\sigma^2)||p(\omega)]                                            \\
      & -\int q(\omega|\mu,\sigma^2)(log(p(\mathcal{D}_{\rm{ID}}|\omega))                                                                               \\
      & -\alpha\cdot log(p(\mathcal{D}_{\rm{OOD}}|\omega)))\,d\omega                                                                                    \\
    = & \mathop{\min}\limits_{\mu,\sigma^2}\mathbb{D}_{KL}[q(\omega|\mu,\sigma^2)||p(\omega)]                                                           \\
      & -\int q(\omega|\mu,\sigma^2)log(\frac{p(\mathcal{D}_{\rm{ID}}|\omega)}{p(\mathcal{D}_{\rm{OOD}}|\omega)^\alpha})^{\frac{1}{1-\alpha}}\,d\omega.
\end{align}

The second term can be interpreted as a transformation of initial loss terms, and if we assume $p(D|\omega)=(\frac{p(\mathcal{D}_{\rm{ID}}|\omega)}{p(\mathcal{D}_{\rm{OOD}}|\omega)^\alpha})^{\frac{1}{1-\alpha}}, \alpha \in (0,1)$,the objective functions can be reinterpreted as
\begin{align}
      & \mathop{\min}\limits_{\mu,\sigma^2}\mathbb{D}_{KL}[q(\omega|\mu,\sigma^2)||p(\omega|\mathcal{D}_{\rm{ID}})]               \\
      & + \mathop{\max}\limits_{\mu,\sigma^2}\alpha\cdot\mathbb{D}_{KL}[q(\omega|\mu,\sigma^2)||p(\omega|\mathcal{D}_{\rm{OOD}})] \\
    = & \mathop{\min}\limits_{\mu,\sigma^2}\mathbb{D}_{KL}[q(\omega|\mu,\sigma^2)||p(\omega|D)].
\end{align}

\section{Experiments}
\label{experiment}

\begin{table*}[htbp]
    \centering
    \caption{Cluster results on MNIST, SVHN and CIFAR10.}\label{full_result_3_1}
    \begin{small}
        \resizebox{.95\linewidth}{!}{
            \begin{tabular}{l|ccc|ccc|ccc|ccc|ccc|ccc}
                \toprule
                Model    & \multicolumn{3}{c}{BBP}    & \multicolumn{3}{c}{VBOE}   & \multicolumn{3}{c}{SDE-Net} & \multicolumn{3}{c}{OE}     & \multicolumn{3}{c}{WOODS}  & \multicolumn{3}{c}{ABNN}                                                                                    \\
                \midrule
                \multirow{3}{*}{\makecell[l]{Confusion                                                                                                                                                                                                                                   \\ matrix}}& 9000 & 0 & 0 & 9000 & 0 & 0 & 9000 & 0 & 0 & 9000 & 0 & 0 & 9000 & 0 & 0 & 9000 & 0 & 0\\
                ~        & 9000                       & 0                          & 0                           & 5929                       & 3071                       & 0                          & 713 & 320 & 7967 & 9000 & 0    & 0    & 1183 & 7817 & 0   & 5894 & 3106 & 0    \\
                ~        & 3850                       & 2689                       & 2461                        & 1881                       & 2800                       & 4319                       & 9   & 19  & 8972 & 4104 & 2506 & 2390 & 338  & 7797 & 865 & 1253 & 1457 & 6290 \\
                \midrule
                Accuracy & \multicolumn{3}{c}{42.4\%} & \multicolumn{3}{c}{60.7\%} & \multicolumn{3}{c}{67.7\%}  & \multicolumn{3}{c}{42.2\%} & \multicolumn{3}{c}{65.5\%} & \multicolumn{3}{c}{68.1\%}                                                                                  \\
                \hline
            \end{tabular}}
    \end{small}
\end{table*}

\begin{table*}[htbp]
    \centering
    \caption{Cluster results on CIFAR10 and CIFAR100.}\label{full_result_3_2}
    \begin{small}
            \begin{tabular}{l|ccc|ccc|ccc|ccc|ccc|ccc}
                \toprule
                Model    & \multicolumn{3}{c}{BBP}    & \multicolumn{3}{c}{VBOE}   & \multicolumn{3}{c}{SDE-Net} & \multicolumn{3}{c}{OE}     & \multicolumn{3}{c}{WOODS}  & \multicolumn{3}{c}{ABNN}                                                             \\
                \midrule
                \multirow{3}{*}{\makecell[l]{Confusion                                                                                                                                                                                                            \\ matrix}}& 100 & 0 & 0 & 100 & 0 & 0 & 100 & 0 & 0 & 100 & 0 & 0 & 100 & 0 & 0 & 100 & 0 & 0\\
                ~        & 5                          & 37                         & 58                          & 5                          & 95                         & 0                          & 3 & 82 & 15 & 10 & 90 & 0  & 4 & 96 & 0   & 9 & 91 & 0  \\
                ~        & 0                          & 84                         & 16                          & 0                          & 44                         & 56                         & 0 & 46 & 54 & 0  & 38 & 62 & 0 & 0  & 100 & 0 & 32 & 68 \\
                \midrule
                Accuracy & \multicolumn{3}{c}{51.0\%} & \multicolumn{3}{c}{83.7\%} & \multicolumn{3}{c}{68.7\%}  & \multicolumn{3}{c}{84.0\%} & \multicolumn{3}{c}{98.7\%} & \multicolumn{3}{c}{86.3\%}                                                           \\
                \hline
            \end{tabular}
    \end{small}
\end{table*}

\begin{figure*}[htbp]
    \centering
    \subcaptionbox{BBP}[0.24\linewidth]{
        \includegraphics[width=1\linewidth]{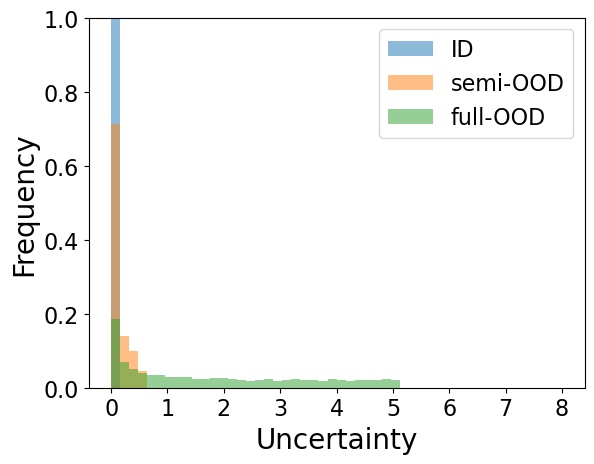}
        \includegraphics[width=1\linewidth]{pic/cluster_visualization_BBP.jpg}
    }
    \subcaptionbox{VBOE}[0.24\linewidth]{
        \includegraphics[width=1\linewidth]{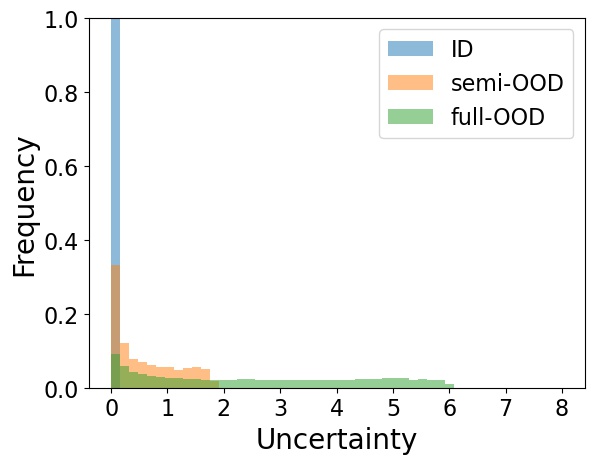}
        \includegraphics[width=1\linewidth]{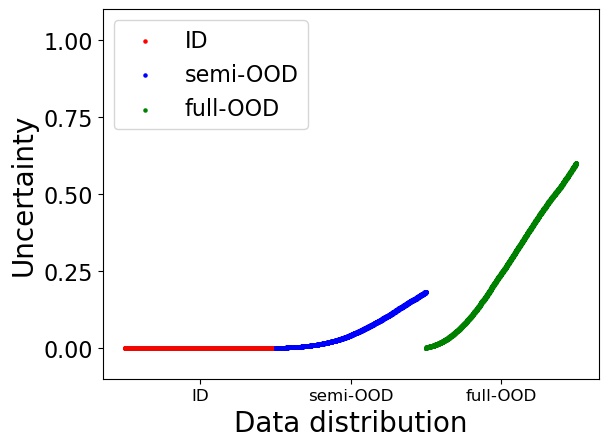}
    }
    \subcaptionbox{SDE-Net}[0.24\linewidth]{
        \includegraphics[width=1\linewidth]{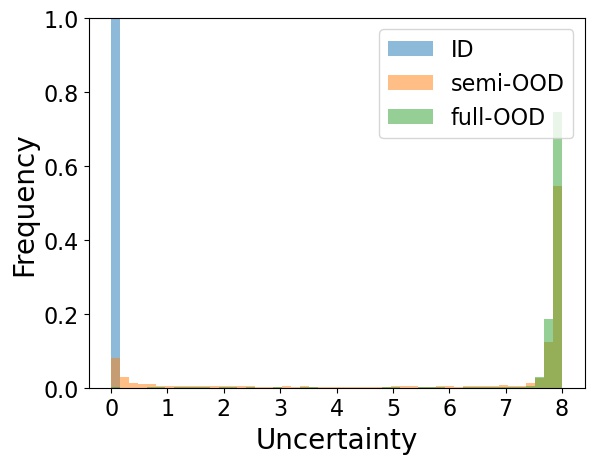}
        \includegraphics[width=1\linewidth]{pic/cluster_visualization_sdenet.jpg}
    }
    \subcaptionbox{OE}[0.24\linewidth]{
        \includegraphics[width=1\linewidth]{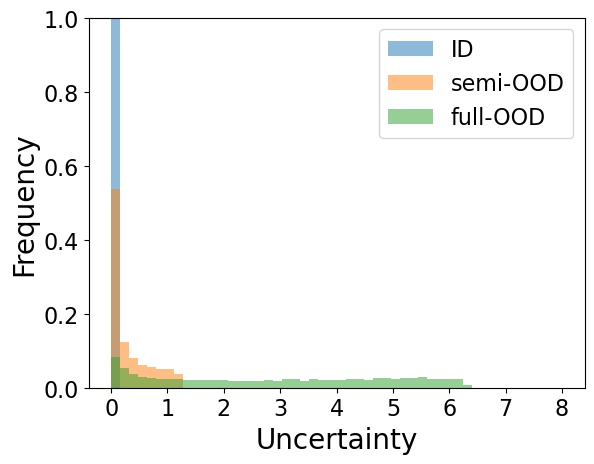}
        \includegraphics[width=1\linewidth]{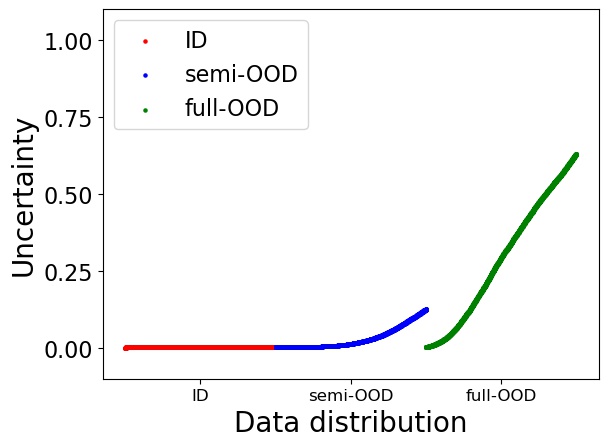}
    }
    \subcaptionbox{WOODS}[0.24\linewidth]{
        \includegraphics[width=1\linewidth]{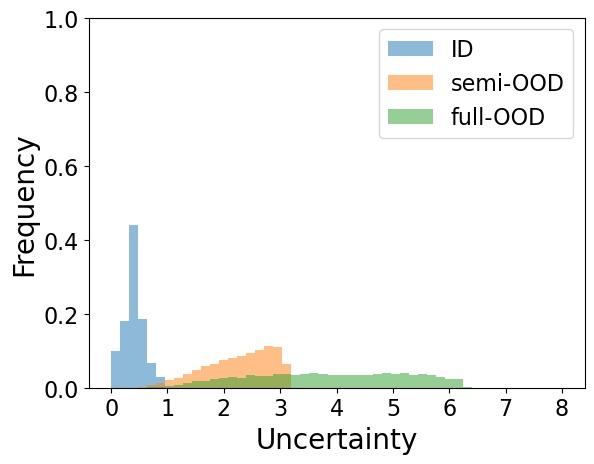}
        \includegraphics[width=1\linewidth]{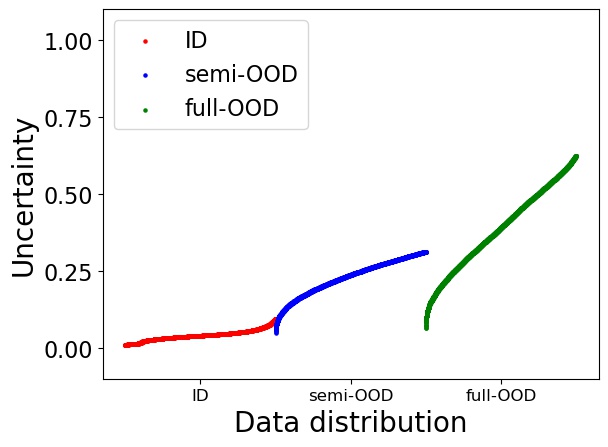}
    }
    \subcaptionbox{ABNN}[0.24\linewidth]{
        \includegraphics[width=1\linewidth]{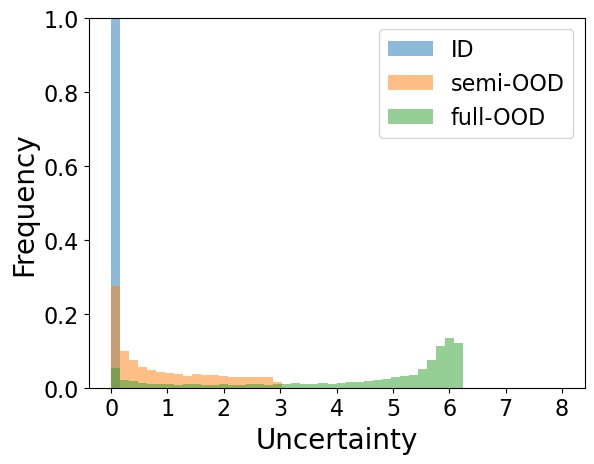}
        \includegraphics[width=1\linewidth]{pic/cluster_visualization_ABNN.jpg}
    }
    \subcaptionbox{Ideal}[0.24\linewidth]{
        \includegraphics[width=1\linewidth]{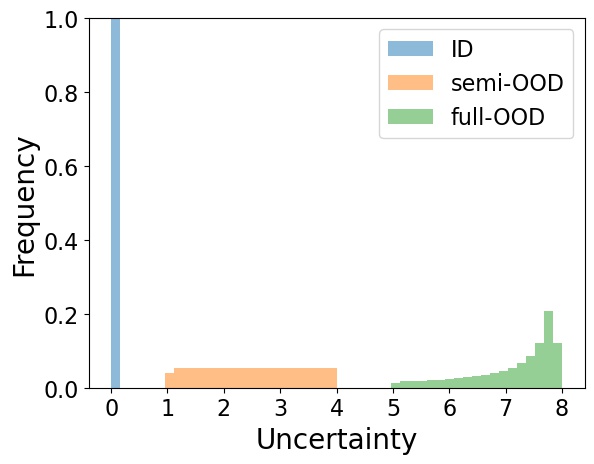}
        \includegraphics[width=1\linewidth]{pic/cluster_visualization_ideal.jpg}
    }
    \caption{Uncertainty distributions on MNIST, SVHN and CIFAR10. The first row is uncertainty distributions; the second row is ordered uncertainty on different datasets. Distributions that are more separated indicate better performances.}
    \label{compare_full_mnist}
\end{figure*}

\begin{figure*}[htbp]
    \centering
    \subcaptionbox{BBP}[0.24\linewidth]{
        \includegraphics[width=1\linewidth]{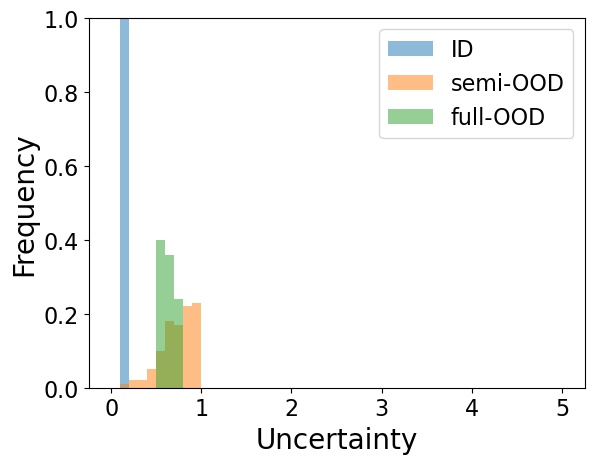}
        \includegraphics[width=1\linewidth]{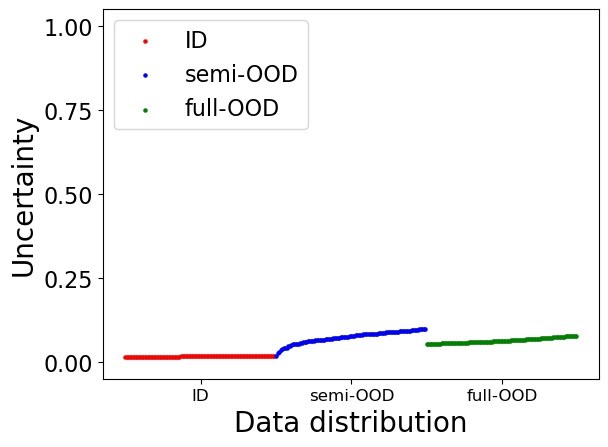}
    }
    \subcaptionbox{VBOE}[0.24\linewidth]{
        \includegraphics[width=1\linewidth]{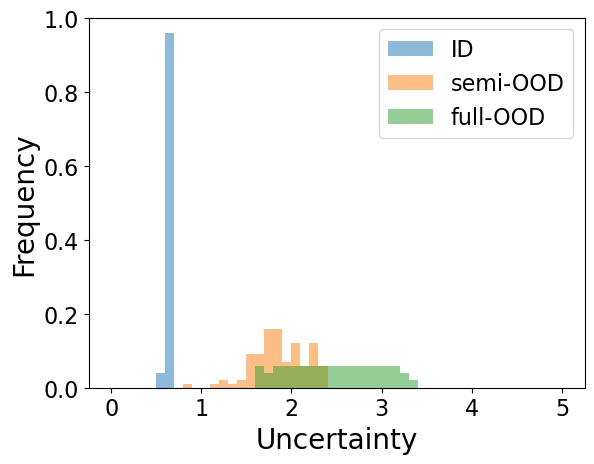}
        \includegraphics[width=1\linewidth]{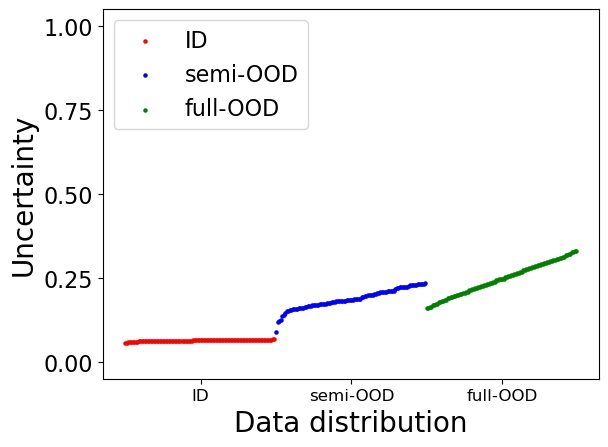}
    }
    \subcaptionbox{SDE-Net}[0.24\linewidth]{
        \includegraphics[width=1\linewidth]{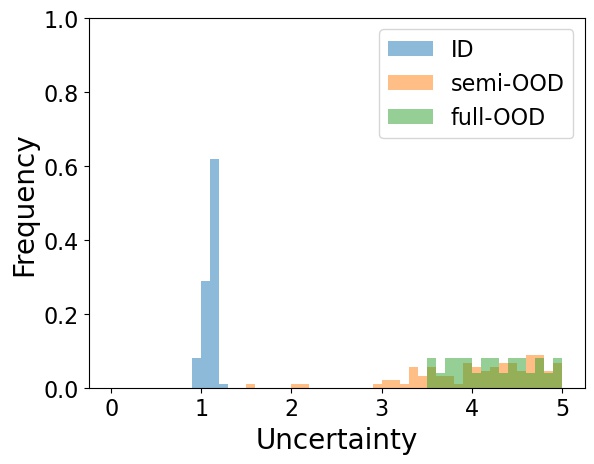}
        \includegraphics[width=1\linewidth]{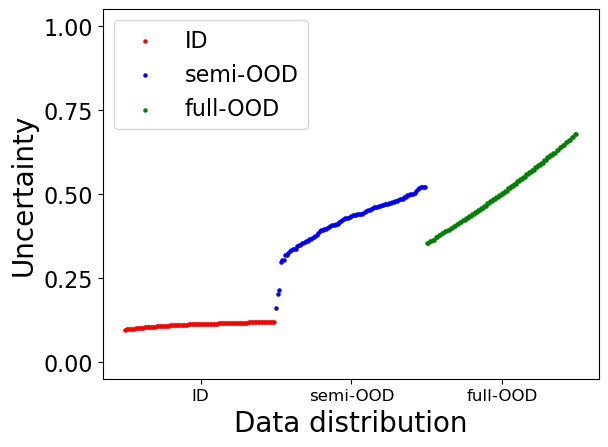}
    }
    \subcaptionbox{OE}[0.24\linewidth]{
        \includegraphics[width=1\linewidth]{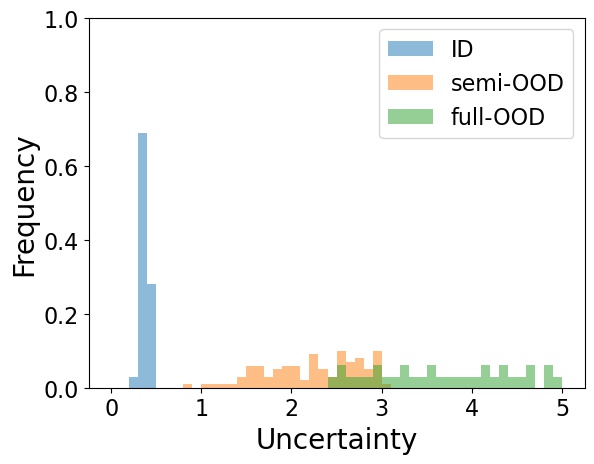}
        \includegraphics[width=1\linewidth]{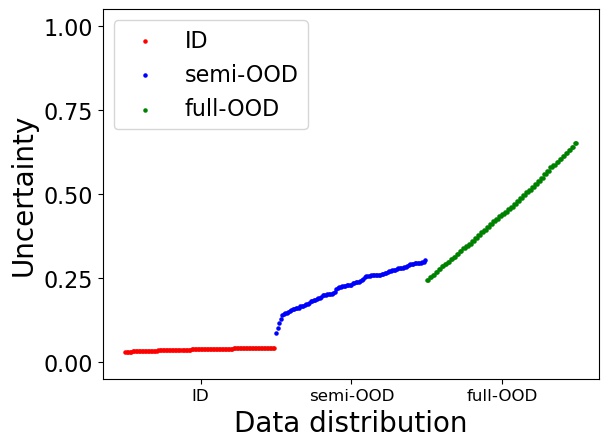}
    }
    \subcaptionbox{WOODS}[0.24\linewidth]{
        \includegraphics[width=1\linewidth]{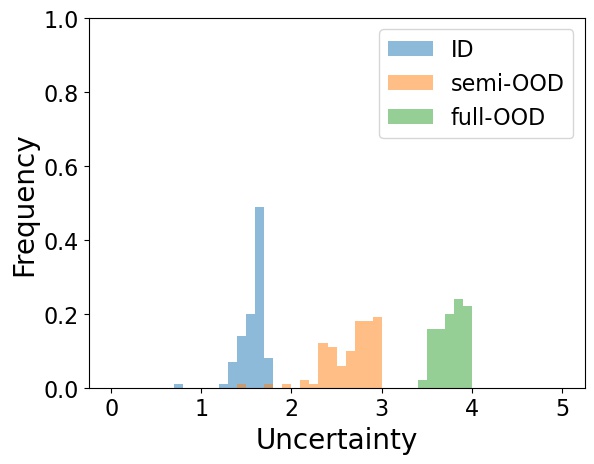}
        \includegraphics[width=1\linewidth]{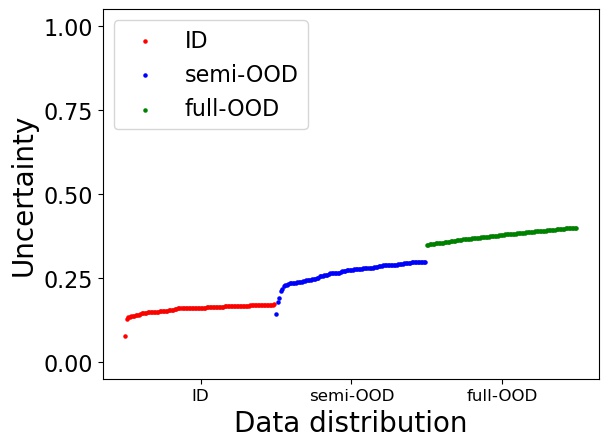}
    }
    \subcaptionbox{ABNN}[0.24\linewidth]{
        \includegraphics[width=1\linewidth]{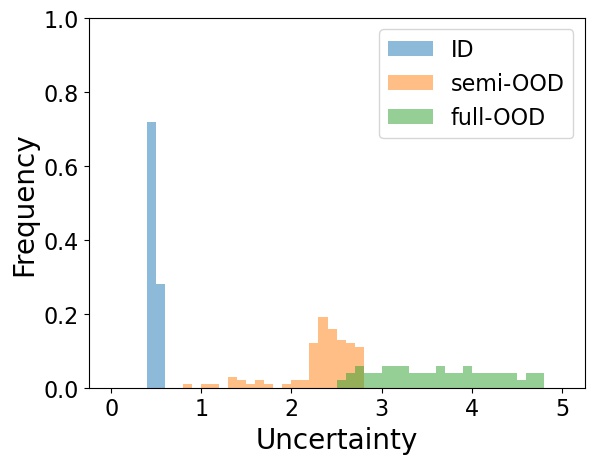}
        \includegraphics[width=1\linewidth]{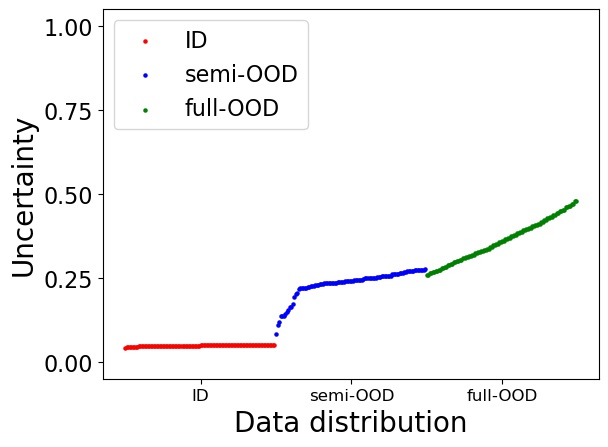}
    }
    \subcaptionbox{Ideal}[0.24\linewidth]{
        \includegraphics[width=1\linewidth]{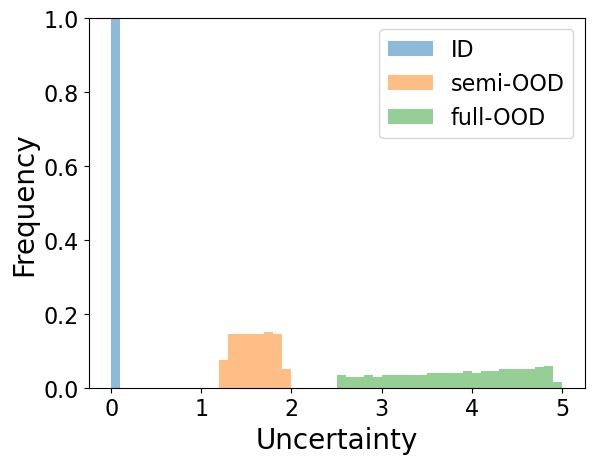}
        \includegraphics[width=1\linewidth]{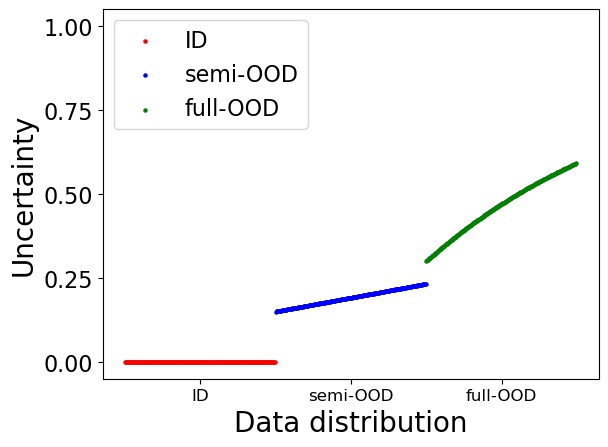}
    }
    \caption{Uncertainty distributions on CIFAR10 and CIFAR100. The first row is uncertainty distributions; the second row is ordered uncertainty on different datasets. Distributions that are more separated indicate better performances.}
    \label{compare_full_cifar}
\end{figure*}

\begin{table*}[htbp]
    \centering
    \caption{Effect of different strategies. We train ABNN on MNIST and take SVHN as OOD data. ``OOD'' means whether OOD data are used; ``attachment'' means whether the structure is an attachment version; ``constant'' means whether the labels are mean values or random variables; ``maximize'' means  whether we maximize KL-divergence on $\mathcal{D}_{\rm{OOD}}$; ``pseudo'' means whether the OOD data is pseudo data or CIFAR10 during training.}
    \begin{small}
        \resizebox{.95\linewidth}{!}{
            \begin{tabular}{cccccc|ccccccr}
                \toprule
                Index & OOD       & Attachment & Constant & Maximize & Pseudo   & \makecell[c]{Classification                                                                             \\accuracy} & \makecell[c]{TNR\\at TPR 95\%} & AUROC & \makecell[c]{Detection\\accuracy} & \makecell[c]{AUPR\\in} & \makecell[c]{AUPR\\out}\\
                \midrule
                (a)   & $\surd$   & $\surd$    & $\surd$  & $\surd$  & $\surd$  & 99.5$\pm$0.0                & 98.3$\pm$1.2  & 99.4$\pm$0.4 & 97.2$\pm$0.9 & 98.7$\pm$0.7 & 99.7$\pm$0.2 \\
                (b)   & $\times $ & $\surd$    & $-$      & $-$      & $-$      & 99.4$\pm$0.0                & 91.7$\pm$2.3  & 97.6$\pm$0.6 & 94.0$\pm$1.1 & 94.8$\pm$2.6 & 98.8$\pm$0.1 \\
                (c)   & $\surd$   & $\times$   & $\surd$  & $\surd$  & $\surd$  & 99.3$\pm$0.0                & 91.6$\pm$1.0  & 97.0$\pm$0.5 & 93.5$\pm$0.5 & 92.4$\pm$1.2 & 99.0$\pm$0.1 \\
                (d)   & $\surd$   & $\surd$    & $\times$ & $-$      & $\surd$  & 99.5$\pm$0.1                & 55.9$\pm$20.4 & 94.2$\pm$2.2 & 92.3$\pm$2.5 & 91.6$\pm$2.5 & 96.7$\pm$1.3 \\
                (e)   & $\surd$   & $\surd$    & $\surd$  & $\times$ & $\surd$  & 99.5$\pm$0.0                & 96.1$\pm$1.1  & 98.9$\pm$0.2 & 95.8$\pm$0.6 & 97.2$\pm$0.7 & 99.6$\pm$0.1 \\
                (f)   & $\surd$   & $\surd$    & $\surd$  & $\surd$  & $\times$ & 99.5$\pm$0.0                & 100.0$\pm$0.0 & 98.7$\pm$0.6 & 99.2$\pm$0.3 & 99.1$\pm$0.4 & 97.6$\pm$1.0 \\
                \bottomrule
            \end{tabular}}
    \end{small}
    \label{result4}
\end{table*}

\begin{table*}[htbp]
    \centering
    \caption{Classification and out-of-distribution detection results. All values are in percentage, and larger values indicates better performance. We mark non-Bayesian models by $\star$ and Bayesian models by $\diamond$. We use \textbf{bold} font to highlight the best results in both groups. ABNN can achieve comparable performances to OOD detection specialized models.}
    \begin{small}
        \resizebox{0.95\linewidth}{!}{
            \begin{tabular}{ccccccccccc}
                \toprule
                ID                   & OOD                     & Model               & Parameters     & \makecell[c]{Classification                                                                                                                          \\accuracy} & \makecell[c]{TNR\\at TPR 95\%} & AUROC & \makecell[c]{Detection\\accuracy} & \makecell[c]{AUPR\\in} & \makecell[c]{AUPR\\out}\\
                \midrule
                \multirow{9}*{MNIST} & \multirow{9}*{SEMEION}  & Threshold$\star$    & 0.58M          & 99.5$\pm$0.0                & 94.0$\pm$1.4          & 98.3$\pm$0.3          & 94.8$\pm$0.7          & 99.7$\pm$0.1           & 89.4$\pm$1.1          \\
                ~                    & ~                       & DeepEnsemble$\star$ & 0.58M$\times$5 & \textbf{99.6$\pm$NA}        & 96.0$\pm$NA           & 98.8$\pm$NA           & 95.8$\pm$NA           & 99.8$\pm$NA            & 91.3$\pm$NA           \\
                ~                    & ~                       & OE$\star$           & 0.58M          & 99.5$\pm$0.1                & 96.3$\pm$0.5          & 98.7$\pm$0.1          & 95.9$\pm$0.2          & 99.8$\pm$0.0           & 90.8$\pm$0.5          \\
                ~                    & ~                       & WOODS$\star$        & 0.58M          & 99.3$\pm$0.0                & 86.4$\pm$0.4          & 97.4$\pm$0.1          & 95.2$\pm$0.2          & 99.6$\pm$0.0           & 87.7$\pm$0.7          \\
                ~                    & ~                       & SDE-Net$\star$      & 0.28M          & 99.4$\pm$0.1                & \textbf{99.6$\pm$0.2} & \textbf{99.9$\pm$0.1} & \textbf{98.6$\pm$0.5} & \textbf{100.0$\pm$0.0} & \textbf{99.5$\pm$0.3} \\
                \cmidrule(r){3-10}
                ~                    & ~                       & BBP$\diamond$       & 0.58M$\times$2 & 99.2$\pm$0.3                & 75.0$\pm$3.4          & 94.8$\pm$1.2          & 90.4$\pm$2.2          & 99.2$\pm$0.3           & 76.0$\pm$4.2          \\
                ~                    & ~                       & p-SGLD$\diamond$    & 0.58M          & 99.3$\pm$0.2                & 85.3$\pm$2.3          & 89.1$\pm$1.6          & 90.5$\pm$1.3          & 93.6$\pm$1.0           & 82.8$\pm$2.2          \\
                ~                    & ~                       & VBOE$\diamond$      & 0.58M$\times$2 & \textbf{99.5$\pm$0.2}       & 96.0$\pm$0.2          & 99.0$\pm$0.0          & 95.4$\pm$0.2          & \textbf{99.9$\pm$0.0}  & 93.6$\pm$0.3          \\
                ~                    & ~                       & ABNN$\diamond$      & 0.58M+0.30M    & \textbf{99.5$\pm$0.0}       & \textbf{97.9$\pm$0.7} & \textbf{99.2$\pm$0.1} & \textbf{96.9$\pm$0.5} & 99.0$\pm$0.0           & \textbf{95.0$\pm$1.2} \\
                \midrule
                \multirow{9}*{MNIST} & \multirow{9}*{SVHN}     & Threshold$\star$    & 0.58M          & 99.5$\pm$0.0                & 90.1$\pm$2.3          & 96.8$\pm$0.9          & 92.9$\pm$1.1          & 90.0$\pm$3.3           & 98.7$\pm$0.3          \\
                ~                    & ~                       & DeepEnsemble$\star$ & 0.58M$\times$5 & \textbf{99.6$\pm$NA}        & 92.7$\pm$NA           & 98.0$\pm$NA           & 94.1$\pm$NA           & 94.5$\pm$NA            & 99.1$\pm$NA           \\
                ~                    & ~                       & OE$\star$           & 0.58M          & 99.5$\pm$0.1                & 92.1$\pm$1.2          & 97.5$\pm$0.4          & 93.7$\pm$0.6          & 93.1$\pm$1.3           & 99.0$\pm$0.2          \\
                ~                    & ~                       & WOODS$\star$        & 0.58M          & 99.3$\pm$0.0                & 94.2z$\pm$1.6         & 98.0$\pm$0.2          & 96.0$\pm$0.2          & 97.6$\pm$0.2           & 98.6$\pm$0.2          \\
                ~                    & ~                       & SDE-Net$\star$      & 0.28M          & 99.4$\pm$0.1                & \textbf{97.8$\pm$1.1} & \textbf{99.5$\pm$0.2} & \textbf{97.0$\pm$0.2} & \textbf{98.6$\pm$0.6}  & \textbf{99.8$\pm$0.1} \\
                \cmidrule(r){3-10}
                ~                    & ~                       & BBP$\diamond$       & 0.58M$\times$2 & 99.2$\pm$0.3                & 80.5$\pm$3.2          & 96.0$\pm$1.1          & 91.9$\pm$0.9          & 92.6$\pm$2.4           & 98.3$\pm$0.4          \\
                ~                    & ~                       & p-SGLD$\diamond$    & 0.58M          & 99.3$\pm$0.2                & 94.5$\pm$2.1          & 95.7$\pm$1.3          & 95.0$\pm$1.2          & 75.6$\pm$5.2           & 98.7$\pm$0.2          \\
                ~                    & ~                       & VBOE$\diamond$      & 0.58M$\times$2 & \textbf{99.5$\pm$0.0}       & 98.1$\pm$0.2          & \textbf{99.5$\pm$0.1} & 97.6$\pm$0.2          & 98.0$\pm$0.1           & \textbf{99.8$\pm$0.0} \\
                ~                    & ~                       & ABNN$\diamond$      & 0.58M+0.30M    & \textbf{99.5$\pm$0.0}       & \textbf{98.3$\pm$1.2} & 99.4$\pm$0.4          & \textbf{99.5$\pm$0.2} & \textbf{98.7$\pm$0.7}  & 99.7$\pm$0.2          \\
                \midrule
                \multirow{9}*{SVHN}  & \multirow{9}*{CIFAR10}  & Threshold$\star$    & 0.58M          & 95.2$\pm$0.1                & 66.1$\pm$1.9          & 94.4$\pm$0.4          & 89.8$\pm$0.5          & 96.7$\pm$0.2           & 84.6$\pm$0.8          \\
                ~                    & ~                       & DeepEnsemble$\star$ & 0.58M$\times$5 & \textbf{95.4$\pm$NA}        & 66.5$\pm$NA           & 94.6$\pm$NA           & 90.1$\pm$NA           & 97.8$\pm$NA            & 84.8$\pm$NA           \\
                ~                    & ~                       & OE$\star$           & 0.58M          & 95.2$\pm$0.0                & 67.8$\pm$2.8          & 94.9$\pm$0.5          & 90.2$\pm$0.5          & 98.0$\pm$0.2           & 85.8$\pm$1.3          \\
                ~                    & ~                       & WOODS$\star$        & 0.58M          & 94.3$\pm$0.1                & 57.9$\pm$1.5          & 94.1$\pm$0.2          & 88.6$\pm$0.2          & 97.8$\pm$0.1           & 82.7$\pm$0.4          \\
                ~                    & ~                       & SDE-Net$\star$      & 0.32M          & 94.2$\pm$0.2                & \textbf{87.5$\pm$2.8} & \textbf{97.8$\pm$0.4} & \textbf{92.7$\pm$0.7} & \textbf{99.2$\pm$0.2}  & \textbf{93.7$\pm$0.9} \\
                \cmidrule(r){3-10}
                ~                    & ~                       & BBP$\diamond$       & 0.58M$\times$2 & 93.3$\pm$0.6                & 42.2$\pm$1.2          & 90.4$\pm$0.3          & 83.9$\pm$0.4          & 96.4$\pm$0.2           & 73.9$\pm$0.5          \\
                ~                    & ~                       & p-SGLD$\diamond$    & 0.58M          & 94.1$\pm$0.5                & 63.5$\pm$0.9          & 94.3$\pm$0.4          & 87.8$\pm$1.2          & 97.9$\pm$0.2           & 83.9$\pm$0.7          \\
                ~                    & ~                       & VBOE$\diamond$      & 0.58M$\times$2 & 95.1$\pm$0.2                & 61.8$\pm$1.1          & 93.7$\pm$0.3          & 88.7$\pm$0.4          & 97.3$\pm$0.2           & 83.0$\pm$0.6          \\
                ~                    & ~                       & ABNN$\diamond$      & 0.58M+0.30M    & \textbf{95.3$\pm$0.3}       & \textbf{72.8$\pm$2.4} & \textbf{96.0$\pm$0.4} & \textbf{90.7$\pm$0.5} & \textbf{98.5$\pm$0.1}  & \textbf{88.5$\pm$1.1} \\
                \midrule
                \multirow{9}*{SVHN}  & \multirow{9}*{CIFAR100} & Threshold$\star$    & 0.58M          & 95.2$\pm$0.1                & 64.6$\pm$1.9          & 93.8$\pm$0.4          & 88.3$\pm$0.4          & 97.0$\pm$0.2           & 83.7$\pm$0.8          \\
                ~                    & ~                       & DeepEnsemble$\star$ & 0.58M$\times$5 & \textbf{95.4$\pm$NA}        & 64.4$\pm$NA           & 93.9$\pm$NA           & 89.4$\pm$NA           & 97.4$\pm$NA            & 84.8$\pm$NA           \\
                ~                    & ~                       & OE$\star$           & 0.58M          & 95.2$\pm$0.0                & 65.3$\pm$2.7          & 94.4$\pm$0.5          & 89.6$\pm$0.4          & 97.7$\pm$0.3           & 84.8$\pm$1.2          \\
                ~                    & ~                       & WOODS$\star$        & 0.58M          & 94.3$\pm$0.1                & 58.8$\pm$1.5          & 93.6$\pm$0.2          & 88.0$\pm$0.1          & 97.6$\pm$0.1           & 81.9$\pm$0.5          \\
                ~                    & ~                       & SDE-Net$\star$      & 0.32M          & 94.2$\pm$0.2                & \textbf{83.4$\pm$3.6} & \textbf{97.0$\pm$0.4} & \textbf{91.6$\pm$0.7} & \textbf{98.8$\pm$0.1}  & \textbf{92.3$\pm$1.1} \\
                \cmidrule(r){3-10}
                ~                    & ~                       & BBP$\diamond$       & 0.58M$\times$2 & 93.3$\pm$0.6                & 42.4$\pm$0.3          & 90.6$\pm$0.2          & 84.3$\pm$0.3          & 96.5$\pm$0.1           & 75.2$\pm$0.9          \\
                ~                    & ~                       & p-SGLD$\diamond$    & 0.58M          & 94.1$\pm$0.5                & 62.0$\pm$0.5          & 91.3$\pm$1.2          & 86.0$\pm$0.2          & 93.1$\pm$0.8           & 81.9$\pm$1.3          \\
                ~                    & ~                       & VBOE$\diamond$      & 0.58M$\times$2 & 95.1$\pm$0.2                & 60.4$\pm$1.1          & 93.3$\pm$0.3          & 88.2$\pm$0.4          & 97.1$\pm$0.2           & 82.3$\pm$0.6          \\
                ~                    & ~                       & ABNN$\diamond$      & 0.58M+0.30M    & \textbf{95.3$\pm$0.3}       & \textbf{70.1$\pm$0.6} & \textbf{95.2$\pm$0.3} & \textbf{91.1$\pm$0.3} & \textbf{99.0$\pm$0.2}  & \textbf{89.2$\pm$0.8} \\
                \bottomrule
            \end{tabular}}
    \end{small}
    \label{result1}
\end{table*}

\begin{table*}[htbp]
    \caption{Misclassification detection performance. We report the average performance and standard deviation for 5 random initializations.\label{result2}}
    \centering
    \begin{small}
        \begin{tabular}{lcccccccccr}
            \toprule
            Data                  & Model         & \makecell[c]{TNR                                                                                                       \\at TPR 95\%} & AUROC & \makecell[c]{Detection\\accuracy} & \makecell[c]{AUPR\\succ} & \makecell[c]{AUPR\\err}\\
            \midrule
            \multirow{10}*{MNIST} & Threshold     & 85.4$\pm$2.8          & 94.3$\pm$0.9          & 92.1$\pm$1.5          & 99.8$\pm$0.1           & 31.9$\pm$8.3          \\
            ~                     & DeepEnsemble  & 89.6$\pm$NA           & 97.5$\pm$NA           & 93.2$\pm$NA           & \textbf{100.0$\pm$NA}  & 41.4$\pm$NA           \\
            ~                     & OE            & 88.4$\pm$2.1          & 96.2$\pm$0.5          & 92.8$\pm$0.9          & \textbf{100.0$\pm$0.0} & 32.9$\pm$7.9          \\
            ~                     & WOODS         & 85.1$\pm$0.8          & 97.0$\pm$0.1          & 92.3$\pm$0.5          & 99.9$\pm$0.0           & \textbf{43.2$\pm$3.1} \\
            ~                     & SDE-Net       & 88.5$\pm$1.3          & 96.8$\pm$0.9          & 92.9$\pm$0.8          & \textbf{100.0$\pm$0.0} & 36.6$\pm$4.6          \\
            ~                     & BBP           & 88.7$\pm$0.9          & 96.5$\pm$2.1          & 93.1$\pm$0.5          & 100.0$\pm$0.0          & 35.4$\pm$3.2          \\
            ~                     & p-SGLD        & 93.2$\pm$2.5          & 96.4$\pm$1.7          & \textbf{98.4$\pm$0.2} & \textbf{100.0$\pm$0.0} & 42.0$\pm$2.4          \\
            ~                     & VBOE          & 85.1$\pm$4.3          & 94.6$\pm$0.9          & 91.5$\pm$1.9          & \textbf{100.0$\pm$0.0} & 31.4$\pm$1.8          \\
            ~                     & ABNN          & 92.4$\pm$2.1          & 98.4$\pm$0.5          & 95.3$\pm$0.7          & 99.7$\pm$0.1           & 36.2$\pm$0.5          \\
            ~                     & ABNN(CIFAR10) & \textbf{93.9$\pm$2.6} & \textbf{98.7$\pm$0.3} & 97.5$\pm$1.2          & \textbf{100.0$\pm$0.0} & 37.8$\pm$4.6          \\
            \midrule
            \multirow{10}*{SVHN}  & Threshold     & 66.4$\pm$1.7          & 90.1$\pm$0.3          & 85.9$\pm$0.4          & 99.3$\pm$0.0           & 42.8$\pm$0.6          \\
            ~                     & DeepEnsemble  & 67.2$\pm$NA           & 91.0$\pm$NA           & 86.6$\pm$NA           & 99.4$\pm$NA            & 46.5$\pm$NA           \\
            ~                     & OE            & 64.5$\pm$1.3          & 91.4$\pm$0.6          & 86.3$\pm$0.6          & 99.4$\pm$0.1           & 45.3$\pm$1.8          \\
            ~                     & WOODS         & 63.9$\pm$1.2          & 92.6$\pm$0.1          & 86.9$\pm$0.3          & 99.4$\pm$0.0           & \textbf48.6$\pm$1.9   \\
            ~                     & SDE-Net       & 65.6$\pm$1.9          & 92.3$\pm$0.5          & 86.8$\pm$0.4          & 99.4$\pm$0.0           & \textbf{53.9$\pm$2.5} \\
            ~                     & BBP           & 58.7$\pm$2.1          & 91.8$\pm$0.2          & 85.6$\pm$0.7          & 99.1$\pm$0.1           & 50.7$\pm$0.9          \\
            ~                     & p-SGLD        & 64.2$\pm$1.3          & 93.0$\pm$0.4          & 87.1$\pm$0.4          & 99.4$\pm$0.1           & 48.6$\pm$1.8          \\
            ~                     & VBOE          & 63.4$\pm$1.0          & 90.5$\pm$0.1          & 84.8$\pm$0.3          & 99.3$\pm$0.0           & 44.6$\pm$1.5          \\
            ~                     & ABNN          & 67.1$\pm$2.8          & 92.9$\pm$0.3          & 86.7$\pm$0.3          & 99.4$\pm$0.1           & 50.2$\pm$1.8          \\
            ~                     & ABNN(CIFAR10) & \textbf{67.3$\pm$1.8} & \textbf{93.2$\pm$0.2} & \textbf{87.9$\pm$0.3} & \textbf{99.5$\pm$0.1}  & 52.4$\pm$1.8          \\
            \bottomrule
        \end{tabular}
    \end{small}
\end{table*}

In this section, we evaluate the uncertainty estimation ability of ABNN through five sets of experiments. All experiments share the same setup, where the backbone is ResNet18 and the models are trained using Adam optimizer. The hyperparameters of all methods are set according to their default values and the $\alpha$ for ABNN is 0.95. The influence of varying $\alpha$ is shown in \cref{alpha}. Additionally, all models are trained by the same pseudo OOD data (constructed by adding Gaussian noise to ID data, without any additional information), which means all methods  have access to real OOD data before testing. To ensure consistency, the iteration times of SDE-Net are aligned with the number of residual blocks of ABNN and other methods. For methods that need sampling, we sample 5 times at each training step.

In \cref{4_2}, we compare ABNN with the following uncertainty estimate or OOD detection methods: (1)BBP, a classical variational inference BNN \cite{BBB}; (2)VBOE, a BNN that includes OOD training like OE \cite{kristiadi2022being}; (3) SDE-Net, a OOD detection specified method which shares similar structure and training process with ABNN \cite{kong2020sde}; (4)OE, one of the most typical OOD training method \cite{hendrycks2018deep}; (5)WOODS, a variation of OE and get state of the art OOD detection performance \cite{OE}.

Additionally, in \cref{application}, we expand our comparisons to include: (1) Threshold \cite{hendrycks2016baseline}, (2) DeepEnsemble \cite{deepensemble}, (3) OE \cite{hendrycks2018deep}, (4) WOODS \cite{OE}, (5)SDE-Net \cite{kong2020sde}, (6) BBP, (7) p-SGLD \cite{li2016preconditioned}, a Markov chain Monte Carlo BNN, and (8) VBOE. We follow Hendrycks' metrics \cite{hendrycks2016baseline}.

\subsection{Quantification Experiment}
\label{4_2}

In \cref{3_2}, we provide a brief explanation that ABNN can describe the continued growth of the uncertainty, and here we prove it with experiments. In real-world scenarios, a model can still work on semi-OOD data with a declined accuracy, while losing power on full-OOD data. Therefore, it is necessary to assess a model's ability to differentiate between ID, semi-OOD and full-OOD data.

We show uncertainty by the max value of the prediction vectors after softmax layers, which is usually seen as the probability that the input belongs to its respective class. However, for triple classification tasks like ID, semi-OOD, and full-OOD, there is no widely adopted metric. Drawing inspiration from unsupervised learning methodologies \cite{caron2018deep,dai2022cluster}, we propose a metric based on clustering. In this approach, ID, semi-OOD, and full-OOD can be viewed as three distinct clusters. A proficient uncertainty estimation model would effectively cluster predictions across the entire data space into these three groups.

We set two experiments. In the first experiment, we assign MNIST as ID data, SVHN as semi-OOD data, and CIFAR10 as full-OOD data; in the second experiment, we build ``Cat-Dog'', ``Tiger-Wolf'' datasets from CIFAR10 and CIFAR100 to serve as ID and semi-OOD data, and use other unrelated CIFAR images as full-OOD data. For a comprehensive evaluation, we utilized confusion matrices and density histograms to show the clustering results. To avoid taking the worst ID results and the best full-OOD results for unfair clustering, we remove predictions with extreme uncertainty from each dataset.

The quantitative results are listed in \cref{full_result_3_1} and \cref{full_result_3_2}. We report the average performance for 5 random initializations. As shown, all methods can treat ID data well. Since BBP only uses ID data for training, it can hardly recognize semi-OOD data and even full-OOD data. Besides BBP and WOODS, all the other methods can somehow recognize semi-OOD data and full-OOD data, but they still recognize many semi-OOD data either as ID data or full-OOD data. Among them, ABNN get the best overall results. WOODS achieve a perfect result in \cref{full_result_3_2}, but its classification accuracy in original task is only 64.7\% on training set (92.5\% for ABNN), which means WOODS may sacrifice prediction ability.

The qualitative results are show in \cref{compare_full_mnist} and \cref{compare_full_cifar}. We estimate uncertainty distributions by frequency in the top row and list all predictions by order to show the increasing trend of uncertainty in the bottom row. The Ideal visualization is draw manually, following the rule that uncertainty should increase gradually but be separable on semi-OOD and full-OOD datasets. As shown in \cref{compare_full_mnist}, besides ABNN, SDE-Net and WOODS, all methods give too low uncertainty estimation for semi-OOD data, and SDE-Net gives too high uncertainty estimate for semi-OOD data. Although WOODS can separate ID, semi-OOD and full-OOD data well, it still gives too many low estimates for full-OOD data and gives too high estimates for ID data. It is surprising that OE methods still gives low uncertainty estimation for semi-OOD data, and we find this is because we use pseudo OOD data(ID data added noises) for OOD training. If we use true OOD data or pseudo data with greater noises, their visualizations will be like SDE-Net's. As shown in \cref{compare_full_cifar}, BBP and SDE-Net mix up semi-OOD and full-OOD data. VBOE and OE can somehow distinguish semi-OOD and full-OOD data, but there are still overlapped uncertainty. WOODS can well distinguish different data, but it gives too high estimation for ID uncertainty. Although there are some mixed data, ABNN treat different uncertainties better than other models do.

\subsection{Ablation Experiment}

We also conduct ablation studies to examine the contributions of different aspects of our method and identify any unnecessary strategies. We divide our method into five strategies: (1) whether OOD data are used; (2) whether the structure is an attachment version; (3) whether the labels are random versions or constant versions; (4)if labels are constant versions, whether we maximize KL-divergence or minimize KL-divergence on $\mathcal{D}_{\rm{OOD}}$; (5) whether the OOD data is from true dataset or pseudo ones.

Following previous work \cite{hendrycks2016baseline}, we use (1) true negative rate (TNR) at 95\% true positive rate
(TPR); (2) area under the receiver operating characteristic
curve (AUROC); (3) area under the precision-recall curve
(AUPR); and (4) detection accuracy to evaluate performance. Larger values indicate better detection performance. We take MNIST as ID data and SVHN as OOD data. We report the average performance and standard deviation for 5 random initializations. The results are shown in \cref{result4}.

According to row (a) and row (b), adding OOD data during training time does help ABNN to catch more uncertainty. Row (a) and row(c) suggest that without an attachment structure, ABNN will not only harm the classification accuracy but also have worse uncertainty estimation ability. Row (d) shows that if we use random variables as training labels for OOD data, the training process will be instable, and we will get a model with poor uncertainty estimation ability. Row (a) and row (e) verify our assumption that maximizing the KL-divergence on OOD data can converge and make a more powerful model. According to row (f), we find that using true datasets as OOD data during training can bring some improvements, but it also brings some drops under several metrics. Compared with results in \cref{result1}, we find that ABNN outperforms many other models overall. In real-world scenes, one can choose how to build $\mathcal{D}_{\rm{OOD}}$ by convenience. In conclusion, the combination of strategies emerges as imperative for effective uncertainty estimation.

\subsection{Applications\label{application}}

\subsubsection{OOD Detection}

OOD detection is a significant application of uncertainty estimation \cite{cai2023out,zhang2021understanding}. In real-world scenarios, when OOD samples are presented to a model, e.g. give a dog image to an MNIST classification model, we hope the model says 'I don't know' instead of making a prediction blindly. A neural network can distinguish ID and OOD data by uncertainty estimation: ID data usually bring lower uncertainty while OOD data tend to present higher uncertainty.

We set 4 groups of experiments by choosing different datasets to be ID and OOD data: (1) MNIST vs SEMEION, (2) MNIST vs SVHN, (3) SVHN vs CIFAR10, (4) SVHN vs CIFAR100. As shown in \cref{result1}, ABNN achieves the best results compared with traditional BNNs. Although ABNN is not designed specifically for OOD detection, it is still comparable with SDE-Net and even outperforms under some metrics. Besides SDE-Net, ABNN is far better than traditional BNNs and other OOD detection models.

Classification accuracy is also an important metric. As shown in \cref{result1}, traditional BNNs may harm the predictive power, but ABNN can still make accurate classifications. We also count the amount of parameters in \cref{result1}. Variational inference BNNs double the amount of parameters. P-SGLD has to store copies of the parameters for evaluation, which is prohibitively costly. Although ABNN is a Bayesian method, it just needs a few more parameters due to the attachment structure.

\subsubsection{Misclassification Detection}

Misclassification detection is another important application of uncertainty estimation \cite{granese2021doctor,sensoy2021misclassification}. Similar to OOD detection, outputs with low uncertainty are more likely to be classified correctly and those with high uncertainty are probably misclassified. As the use of real OOD data during training does not impact fairness, we also evaluate the performance of ABNN trained with CIFAR10 as OOD data.

As shown in \cref{result2}, if not trained with real OOD data, p-SGLD achieves the best performance overall, and ABNN is comparable to it. However, if we fully explore the potential of ABNN by using real-world OOD training data, ABNN gets the best results.

\subsection{ABNN with different backbone\label{backbone}}

We opted for ResNet as our backbone for two main reasons: (1) its ease of understanding, and (2) the fact that some of our competitor models rely on ResNet as their backbone. However, our theories and optimization procedures are unrelated with the network architectures. Therefore, ABNN can be readily adapted to other BNN structures. While the backbone network primarily handles classification tasks, the uncertainty estimation predominantly relies on the configuration of the distribution modules rather than the backbone architecture. Furthermore, ABNN offers the advantage that the expectation module can be frozen, simplifying the training process by focusing solely on the small-sized distribution modules. We provide experimental evidence demonstrating ABNN's performance with different backbones in \cref{exp_backbone}

\begin{table}[htbp]
    \caption{Misclassification detection performance on different backbones. We report the average performance and standard deviation for 5 random initializations.\label{exp_backbone}}
    \centering
        \resizebox{1\linewidth}{!}{
            \begin{tabular}{cccccccccr}
                \toprule
                Backbone        & \makecell[c]{Classification                                                                            \\accuracy}        & \makecell[c]{TNR                                                             \\at TPR 95\%} & AUROC & \makecell[c]{Detection\\accuracy} & \makecell[c]{AUPR\\succ} & \makecell[c]{AUPR\\err}\\
                \midrule
                ResNet18        & 99.5$\pm $0.0               & 98.3$\pm$1.2 & 99.4$\pm$0.4 & 99.5$\pm$0.2 & 98.7$\pm$0.7 & 99.7$\pm$0.2 \\
                MobileNet V3    & 99.5$\pm $0.1               & 98.3$\pm$0.9 & 99.6$\pm$0.2 & 99.2$\pm$0.3 & 98.7$\pm$1.1 & 99.8$\pm$0.1 \\
                EfficientNet V2 & 99.6$\pm $0.2               & 98.0$\pm$1.1 & 99.5$\pm$0.3 & 99.2$\pm$0.2 & 97.9$\pm$0.3 & 99.4$\pm$0.6 \\
                SqueezeNet      & 99.5$\pm $0.1               & 98.2$\pm$1.3 & 99.6$\pm$0.3 & 99.5$\pm$0.5 & 98.2$\pm$0.4 & 99.6$\pm$0.1 \\
                \bottomrule
            \end{tabular}}
\end{table}

\subsection{Function of $\alpha$\label{alpha}}

While $\alpha$ appears to balance the uncertainty obtained from ID and OOD data, its true function is to determine the relationship between uncertainty and variance. As shown in \cref{assumption1}, the OOD training attempts to enlarge the variance to $\infty $, whereas ID training imposes constraints on the variance. As a result, there is an adversarial training that finally defines how large the variance can be. This maximum variance serves as a representation of the greatest uncertainty. Therefore, $\alpha$ implicitly defines how much uncertainty is represented by one variance level. We show by experiments that the choice of $\alpha$ would not hurt the uncertainty estimation ability of ABNN in \cref{different_alpha}.

\begin{table}[htbp]
    \caption{Misclassification detection performance with different $\alpha$. We report the average performance and standard deviation for 5 random initializations.\label{different_alpha}}
    \centering
        \resizebox{0.95\linewidth}{!}{
            \begin{tabular}{cccccccccr}
                \toprule
                $\alpha$ & \makecell[c]{Classification                                                                            \\accuracy} & \makecell[c]{TNR                                                             \\at TPR 95\%} & AUROC & \makecell[c]{Detection\\accuracy} & \makecell[c]{AUPR\\succ} & \makecell[c]{AUPR\\err}\\
                \midrule
                0.5      & 99.5$\pm $0.0               & 98.8$\pm$1.8 & 99.2$\pm$0.4 & 98.2$\pm$0.9 & 97.7$\pm$1.0 & 99.4$\pm$0.2 \\
                0.7      & 99.5$\pm $0.0               & 98.0$\pm$0.9 & 98.8$\pm$0.3 & 99.2$\pm$0.5 & 97.6$\pm$0.5 & 99.4$\pm$0.3 \\
                0.9      & 99.5$\pm $0.0               & 98.4$\pm$1.1 & 99.4$\pm$0.3 & 99.3$\pm$0.8 & 98.7$\pm$0.7 & 99.7$\pm$0.2 \\
                0.95     & 99.5$\pm $0.0               & 98.3$\pm$1.2 & 99.4$\pm$0.4 & 99.5$\pm$0.2 & 98.7$\pm$0.7 & 99.7$\pm$0.2 \\
                1        & 99.5$\pm $0.0               & 98.4$\pm$0.8 & 99.5$\pm$0.3 & 99.4$\pm$0.7 & 98.8$\pm$0.5 & 99.8$\pm$0.2 \\
                \bottomrule
            \end{tabular}}
\end{table}

\section{Discussions}
Our segmentations and definitions operate under the assumption that the classification task aligns with the estimation of the target distribution for each class. However, we have not provided a method to quantify the validity of uncertainty estimation during training for ABNN. This issue could potentially be addressed by analyzing the second loss term, $\mathbb{D}_{KL}[q(\omega|\mu,\sigma^2)||p(\omega|\mathcal{D}_{\rm{ID}})]$. Without the third loss term, the ID KL-divergence should be optimized to its minimum, which implies that the distribution modules are the best estimations of the ID posteriors. However, the introduction of the third loss term may increase certain variances, thereby affecting the ID KL-divergence. Consequently, an increase in the ID KL-divergence signifies the extent to which OOD uncertainty is captured and whether ABNN can maintain its performance on ID data.

\section{Conclusion}

We propose a variational inference Bayesian Neural Network with an attachment structure to catch more uncertainty from OOD data. We provide mathematical descriptions for OOD data and design the attachment structure for proper integration of uncertainty. The convergence of ABNN is theoretically analyzed. The experiments show the superiority of ABNN over traditional BNNs. In the future, we will pay attention to developing a more effective method to catch the uncertainty from OOD data. Overall, our proposed method can be considered as a framework, where the attachment structure can be extended to other general BNNs.

\bibliography{ABNN}
\bibliographystyle{unsrt}

\end{document}